\newif\ifsupp 
\newif\ifcomments  
\newif\ifaistats 

\supptrue
\commentsfalse
\aistatsfalse

\ifaistats
\documentclass[twoside]{article}
\usepackage[accepted]{aistats2020}
\else
\documentclass{article}
\usepackage{fullpage}
\fi

\usepackage{lmodern}
\usepackage{natbib}
\usepackage[utf8]{inputenc} 
\usepackage[T1]{fontenc}    
\usepackage{hyperref}       
\usepackage{url}            
\usepackage{booktabs}       
\usepackage{amsfonts}       
\usepackage{nicefrac}       
\usepackage{microtype}      
\usepackage{xcolor}
\definecolor{DarkGreen}{rgb}{0.1,0.5,0.1}
\definecolor{DarkRed}{rgb}{0.5,0.1,0.1}
\definecolor{DarkBlue}{rgb}{0.1,0.1,0.5}
\usepackage{paralist}
\hypersetup{
    unicode=false,          
    pdftoolbar=true,        
    pdfmenubar=true,        
    pdffitwindow=false,      
    pdftitle={},    
    pdfauthor={}
    pdfsubject={},   
    pdfnewwindow=true,      
    pdfkeywords={keywords}, 
    colorlinks=true,       
    linkcolor=DarkRed,          
    citecolor=DarkGreen,        
    filecolor=DarkRed,      
    urlcolor=DarkBlue,          
}
\usepackage{amsmath, amssymb, amsthm}
\usepackage{cleveref}
\usepackage{algorithm}
\usepackage[noend]{algorithmic}
\usepackage{color}
\usepackage{comment}

\usepackage{caption}
\usepackage{subcaption}
\usepackage{tikz}
\usepackage{tikz-qtree}

\newcommand{\comm}[1]{\textcolor{red}{\small  //#1}}

\newcommand{\var}[1]{\mathbb{V}\left[#1\right]}
\newcommand{\mypar}[1]{\noindent{\bf{#1}:}}

\newcommand\R{\mathbb{R}}
\newcommand{\cA}{\mathcal{A}}

\newcommand{\cD}{\mathcal{D}}

\newcommand{\cM}{\mathcal{M}}

\newcommand{\cO}{\mathcal{O}}

\newcommand{\cQ}{\mathcal{Q}}

\newcommand{\cW}{\mathcal{W}}
\newcommand{\cX}{\mathcal{X}}
\newcommand{\cY}{\mathcal{Y}}

\newcommand{\bX}{\pmb{X}}
\newcommand{\vX}{\vec{X}}
\newcommand{\vbX}{\vec{\bX}}

\newcommand{\MutInfo}[2]{I\left( #1;#2\right)}

\newcommand\numberthis{\addtocounter{equation}{1}\tag{\theequation}}

\newcommand{\PL}[2]{\mathrm{PrivLoss}\left(#1||#2 \right)}

\DeclareMathOperator*{\myargmax}{\arg\!\max}

\newcommand{\bba}{\pmb{a}}

\newcommand{\ex}[1]{\mathbb{E}\left[#1\right]}
\DeclareMathOperator*{\Expectation}{\mathbb{E}}
\newcommand{\Ex}[2]{\Expectation_{#1}\left[#2\right]}
\DeclareMathOperator*{\Probability}{\mathrm{Pr}}
\newcommand{\prob}[1]{\mathrm{Pr}\left[#1\right]}
\newcommand{\Prob}[2]{\Probability_{#1}\left[#2\right]}

\newcommand{\eps}{\varepsilon}

\newcommand{\Lap}{\mathrm{Lap}}
\renewcommand{\hat}{\widehat}

\newcommand{\ML}{\cM_{\Lap}}
\newcommand{\MG}{\cM_{\mathrm{Gauss}}}
\newcommand{\tol}{\tau}
\newcommand{\alg}{\cM}
\newcommand{\adv}{\cA}

\newtheorem{theorem}{Theorem}[section]
\newtheorem{lemma}[theorem]{Lemma}

\newtheorem{corollary}[theorem]{Corollary}

\theoremstyle{definition}
\newtheorem{definition}[theorem]{Definition}

\theoremstyle{remark}

\makeatletter
\def\blfootnote{\gdef\@thefnmark{}\@footnotetext}
\makeatother

\begin{document}

\ifaistats
\runningauthor{Rogers, Roth, Smith, Srebro, Thakkar, Woodworth}
\runningtitle{Guaranteed Validity for Empirical Approaches to Adaptive Data Analysis}
\twocolumn[

\aistatstitle{Guaranteed Validity for Empirical Approaches to \\ Adaptive Data Analysis}

\aistatsauthor{ Ryan Rogers \\ Previously at University of Pennsylvania \And Aaron Roth \\ University of Pennsylvania \And  Adam Smith \\ Boston University \AND Nathan Srebro \\ Toyota Technical Institute of Chicago \And Om Thakkar \\ Boston University \And Blake Woodworth \\ Toyota Technical Institute of Chicago}
\aistatsaddress{} ]
\else
\title{Guaranteed Validity for Empirical Approaches to \\ Adaptive Data Analysis\blfootnote{Accepted to appear in the proceedings of the 23\textsuperscript{rd} International Conference on Artificial
  Intelligence and Statistics (AISTATS) 2020,  Palermo, Italy. PMLR: Volume  108.}}
\author{
Ryan Rogers\thanks{Previously at University of Pennsylvania. \texttt{rrogers386@gmail.com}} \\ Nathan Srebro\thanks{Toyota Technical Institute of Chicago. \texttt{\{nati, blake\}@ttic.edu}}
\and
Aaron Roth\thanks{University of Pennsylvania. \texttt{aaroth@cis.upenn.edu}} \\ Om Thakkar\footnotemark[4]
\and
Adam Smith\thanks{Boston University. \texttt{\{ads22, omthkkr\}@bu.edu}} 
\\ Blake Woodworth\footnotemark[2]
}
\maketitle
\fi

\begin{abstract}
   We design a general framework for answering adaptive statistical queries that focuses on providing explicit confidence intervals along with point estimates. Prior work in this area has either focused on providing tight confidence intervals for specific analyses, or providing general worst-case bounds for point estimates. Unfortunately, as we observe, these worst-case bounds are loose in many settings --- often not even beating simple baselines like sample splitting. Our main contribution is to design a framework for providing valid, instance-specific confidence intervals for point estimates that can be generated by heuristics. When paired with good heuristics, this method gives guarantees that are orders of magnitude better than the best worst-case bounds. We provide a Python library implementing our method.
\end{abstract}
\section{Introduction}
\label{sec:intro}

Many data analysis workflows are \emph{adaptive}, i.e., they re-use data over the course of a sequence of analyses, where the choice of analysis at any given stage depends on the results from previous stages. Such adaptive re-use of data is an important source of \emph{overfitting} in machine learning and \emph{false discovery} in the empirical sciences \citep{GL14}. Adaptive workflows arise, for example, 
when exploratory data analysis is mixed with confirmatory data analysis, when hold-out sets are re-used to search through large hyper-parameter spaces or to perform feature selection, and when  datasets are repeatedly re-used within a research community.  

A simple solution to this problem---that we can view as a naïve benchmark---is to simply not re-use data. More precisely, one could use \emph{sample splitting}: partitioning the dataset into $k$ equal-sized pieces, and using a fresh piece of the dataset for each of $k$ adaptive interactions with the data. This allows us to treat each analysis as nonadaptive, permitting many quantities of interest to be accurately estimated with their empirical estimate, and paired with tight confidence intervals that come from classical statistics. However, this seemingly naïve approach is wasteful in its use of data: the sample size needed to conduct a series of $k$ adaptive  analyses grows linearly with $k$. 

A line of recent work \citep{DFHPRR15,DFHPRR15nips,DFHPRR15Science,RZ16,
BNSSSU16,RRST16,FS17,FS17b,XR17,naturalanalyst,modelsimilarity} aims to improve on this baseline by using mechanisms which provide  ``noisy'' answers to queries rather than exact empirical answers. Methods coming from these works require that the sample size grow  proportional to the \emph{square root} of the number of adaptive analyses, dramatically beating the sample splitting baseline asymptotically. Unfortunately, the bounds proven in these papers---even when optimized---only beat the naïve baseline when both the dataset size, and the number of adaptive rounds, are large; see Figure \ref{fig:intro1} (left). 

\begin{figure}[ht]
	\centering
	\begin{tabular}{cc}
	\hspace*{-10pt}
	\begin{minipage}[b]{0.47\columnwidth}
		\includegraphics[width=\textwidth]{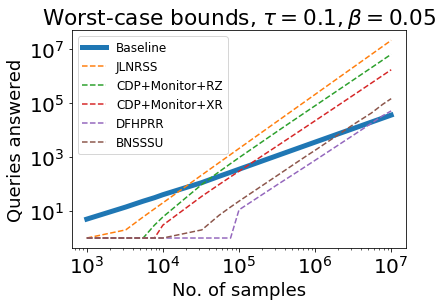}
	\end{minipage}
	\begin{minipage}[b]{0.5\columnwidth}
		\includegraphics[width=\textwidth]{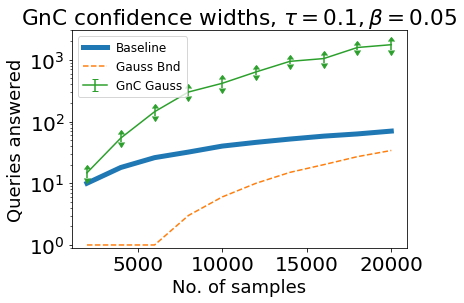}
	\end{minipage}
\end{tabular}
	\caption{ \emph{Left:} Comparison of various worst-case bounds for the Gaussian mechanism with the sample splitting baseline. `DFHPRR' and `BNSSSU' refer to bounds given in prior work (\cite{DFHPRR15STOC,BNSSSU16}). The other lines plot improved worst-case bounds derived in this paper, whereas `JLNRSS' refers to bounds in subsequent work (\cite{JungLNRSS19}). (See Section~\ref{sec:cw_bounds} for the full model and parameter descriptions.)
	\emph{Right:} Performance of  Guess and Check with the Gaussian mechanism providing the guesses (``GnC Gauss'') for a plausible query strategy (see Section~\ref{sec:expts}), compared with the best worst-case bounds for the Gaussian mechanism (``Gauss Bnd''), and the baseline.}
	\label{fig:intro1}
\end{figure}

The failure of these worst-case bounds to beat simple baselines in practice --- despite their attractive asymptotics --- has been a major obstacle to the practical adoption of techniques from this literature. There are two difficulties with directly improving this style of bounds. First, we are limited by what we can prove: mathematical analyses can often be loose by constants that are significant in practice. The more fundamental difficulty is that these bounds are guaranteed to hold even against a \emph{worst-case} data analyst, who is adversarially attempting to find queries which over-fit the sample: one would naturally expect that when applied to a real workload of queries, such worst-case bounds would be extremely pessimistic. We address both difficulties in this paper. 

\paragraph{Contributions}
In this paper, we move the emphasis from algorithms that provide point estimates to algorithms that explicitly manipulate and output \emph{confidence intervals} based on the queries and answers so far, providing the analyst with both an estimated value and a measure of its actual accuracy. At a technical level, we have two types of contributions: 

First, we give optimized worst-case bounds that carefully combine techniques from different pieces of prior work---plotted in Figure \ref{fig:intro1} (left). For certain mechanisms, our improved worst-case bounds are within small constant factors of optimal, in that we can come close to saturating their error bounds with a concrete, adversarial query strategy (Section \ref{sec:cw_bounds}). However, even these optimized bounds require extremely large sample sizes to improve over the naive sample splitting baseline, and their pessimism means they are often loose.

Our main result is the development of a simple framework called \emph{Guess and Check}, that allows an analyst to pair \emph{any method} for ``guessing'' point estimates and confidence interval widths for their adaptive queries, and then rigorously validate those guesses on an additional held-out dataset. So long as the analyst mostly guesses correctly, this procedure can continue indefinitely. The main benefit of this framework is that it allows the analyst to guess confidence intervals whose guarantees \emph{exceed what is guaranteed by the worst-case theory}, and still enjoy rigorous validity in the event that they pass the ``check''. This makes it possible to take advantage of the non-worst-case nature of natural query strategies, and avoid the need to ``pay for'' constants that seem difficult to remove from worst-case bounds. Our empirical evaluation demonstrates that our approach can improve on worst-case bounds by orders of magnitude, and that it improves on the naive baseline even for  modest sample sizes: see Figure \ref{fig:intro1} (right), and Section \ref{sec:gnc} for details. We also provide a Python library 
containing an implementation of our Guess and Check framework.

\paragraph{Related Work}
Our ``Guess and Check'' (GnC) framework draws inspiration from the Thresholdout method of \cite{DFHPRR15nips}, which uses a holdout set in a similar way. GnC has several key differences, which turn out to be crucial for practical performance. 
First, whereas the ``guesses'' in Thresholdout are simply the empirical query answers on a ``training'' portion of the dataset, we make use of other heuristic methods for generating guesses (including, in our experiments, Thresholdout itself) that empirically often seem to prevent overfitting to a substantially larger degree than their worst-case guarantees suggest. 
Second, we make confidence-intervals first-order objects: whereas the ``guesses'' supplied to Thresholdout are simply point estimates, the ``guesses'' supplied to GnC are point estimates along with confidence intervals. 
Finally, we use a more sophisticated analysis to track the number of bits leaked from the holdout, which lets us give tighter confidence intervals and avoids the need to a priori set an upper bound on the number of times the holdout is used.  \cite{GPS18} use a version of Thresholdout to get worst-case accuracy guarantees for values of the AUC-ROC curve for adaptively obtained queries. However, apart from being limited to binary classification tasks and the dataset being used only to obtain AUC values, their bounds require ``unrealistically large'' dataset sizes. Our results are complementary to theirs; by using appropriate concentration inequalities, GnC could also be used to provide confidence intervals for AUC values. Their technique could be used to provide the ``guesses'' to GnC.

Our improved worst-case bounds combine a number of techniques from the existing literature: namely the information theoretic arguments of \cite{RZ16,XR17} together with the ``monitor'' argument of \cite{BNSSSU16}, and a more refined accounting for the properties of specific mechanisms using \emph{concentrated differential privacy} (\cite{DR16,BunS16}). 
\cite{FS17,FS17b} give worst-case bounds that improve with the variance of the asked queries. In Section \ref{sec:expts}, we show how GnC can be used to give tighter bounds when the empirical query variance is small.

\cite{modelsimilarity} give an improved union bound for queries that have high overlap, that can be used to improve bounds for adaptively validating similar models, in combination with description length bounds. \cite{naturalanalyst} take a different approach to going beyond worst-case bounds in adaptive data analysis, by proving bounds that apply to data analysts that may only be adaptive in a constrained way. A difficulty with this approach in practice is that it is limited to analysts whose properties can be inspected and verified --- but provides a potential explanation why worst-case bounds are not observed to be tight in real settings. Our approach is responsive to the degree to which the analyst actually overfits, and so will also provide relatively tight confidence intervals if the analyst satisfies the assumptions of \cite{naturalanalyst}.

In very recent work (subsequent to this paper), \citet{JungLNRSS19} give a further tightening of the worst-case bounds, improving the dependence on the coverage probability $\beta$. Their bounds (shown in Figure~\ref{fig:intro1} (left)) do not significantly affect the comparison with our GnC method since they yield only worst-case analysis.



\subsection{Preliminaries}

As in previous work, we assume that there is a dataset $X = (x_1,\cdots, x_n) \sim \cD^n$ drawn i.i.d. from an unknown distribution $\cD$ over a universe $\cX$. This dataset is the input to a mechanism $\alg$ that also receives a sequence of queries $\phi_1,\phi_2,...$ from an analyst $\adv$ and outputs, for each one, an answer. Each $\phi_i$ is a \emph{statistical query}, defined by a bounded function $\phi_i:\cX\to [0,1]$. We denote the expectation of a statistical query $\phi$ over the data distribution by $\phi(\cD) = \Ex{x \sim \cD}{\phi(x)}$, and the empirical average on a dataset by $\phi(X) = \frac{1}{n} \sum_{i=1}^n \phi(x_i)$. 

The mechanism's goal is to give estimates of $\phi_i(\cD)$ for query $\phi_i$ on the unknown $\cD$. Previous work looked at analysts that produce a single point estimate $a_i$, and measured error based on the distances $\left|a_i - \phi_i(\cD)\right|$. As mentioned above, we propose a shift in focus: we ask mechanisms to produce a confidence interval specified by a point estimate $a_i$ and width $\tau_i$. The answer $(a_i,\tau_i)$ is \emph{correct for $\phi_i$ on $\cD$} if $\phi_i(\cD) \in (a_i - \tau_i, a_i + \tau_i)$. (Note that the data play no role in the definition of correctness---we measure only population accuracy.)

An interaction between randomized algorithms $\alg$ and $\adv$ on dataset $X\in\cX^n$ (denoted $\alg(X) \rightleftharpoons \adv$) consists of an unbounded number of query-answer rounds: at round $i$, $\adv$ sends $\phi_i$, and $\alg(X)$ replies with $(a_i,\tau_i)$. $\alg$ receives $X$ as input. $\adv$ receives no direct input, but may select queries  \emph{adaptively}, based on the answers in previous rounds. The interaction ends when either the mechanism or the analyst stops. We say that the mechanism provides simultaneous coverage if, with high probability,  \emph{all} its answers are correct:

\begin{definition}[Simultaneous Coverage]\label{defn:coverage}
  Given $\beta\in (0,1)$, we say that $\alg$ has \emph{simultaneous coverage $1-\beta$} if, for all $n\in \mathbb{N}$, all distributions $\cD$ on $\cX$ and all randomized algorithms $\adv$, 
  {\small
  \[
    \Pr_{\substack{X\sim\cD^n, \\
        \left\{(\phi_i,a_i, \tau_i)\right\}_{i=1}^k \gets (\alg(X) \rightleftharpoons \adv)}}
    \left[
      \forall i\in [k]: \phi_i(\cD) \in a_i \pm \tau_i
    \right] \geq 1-\beta
  \]
  }
\end{definition}

We denote by $k$ the (possibly random) number of rounds in a given interaction.

\begin{definition}[Accuracy]\label{defn:accuracy}
We say $\alg$ is \emph{$(\tau, \beta)$-accurate}, if $\alg$ has simultaneous coverage $1-\beta$ and its interval widths satisfy $ \max\limits_{i \in [k]} \tau_i\leq \tau$ with probability 1.
\end{definition}

We defer some additional preliminaries to Appendix~\ref{app:defns}.

\section{Confidence intervals from worst-case bounds}
\label{sec:cw_bounds}

Our emphasis on explicit confidence intervals led us to derive worst-case bounds that are as tight as possible given the techniques in the literature. We discuss the Gaussian mechanism here, and defer the application to Thresholdout
\ifsupp 
in \Cref{app:thresh_cw}, and provide a pseudocode for Thresholdout in \Cref{alg:thresh}. 
\fi 

The Gaussian mechanism is defined to be an algorithm that, given input dataset $X \sim \cD^n$ and a query $\phi: \cX \to [0,1]$, reports an answer $a = \phi(X) + N\left(0,\frac{1}{2n^2 \rho}\right)$, where $\rho > 0$ is a parameter. It has existing analyses for simultaneous coverage (see \cite{DFHPRR15STOC,BNSSSU16}) --- but these analyses involve large, sub-optimal constants. Here, we provide an improved worst-case analysis by carefully combining existing techniques.  We use results from \cite{BS16} to bound the mutual information of the output of the Gaussian mechanism with its input.  We then apply an argument similar to that of \cite{RZ16} to bound the bias of the empirical average of a statistical query selected as a function of the perturbed outputs.  Finally, we use Chebyshev's inequality, and the monitor argument from \cite{BNSSSU16} 
to obtain high probability accuracy bound. Figure~\ref{fig:intro1} shows the improvement in the number of queries that can be answered with the Gaussian mechanism with $(0.1,0.05)$-accuracy.
Our guarantee is stated below, with its proof deferred to 
\ifsupp
Appendix~\ref{app:proofRZ_CW}.
\fi
\begin{theorem}
Given input $X \sim \cD^n$, confidence parameter $\beta$, and parameter $\rho$, the Gaussian mechanism is $(\tol, \beta$)-accurate, where
$
\tol =  \sqrt{\frac{1}{2 n \beta } \cdot \min\limits_{\lambda \in [0,1)} \left(\frac{ 2 \rho k n - \ln \left( 1-\lambda \right)}{\lambda}\right)} + \frac{1}{2n} \sqrt{\frac{1}{\rho}\ln\left(\frac{4k}{\beta}\right)}.
$
\label{thm:RZ_CW}
\end{theorem}

We now consider the extent to which our analyses are improvable for worst-case queries to the Gaussian and the Thresholdout mechanisms.  To do this, we derive the worst query strategy in a particular restricted regime. We call it the ``single-adaptive query strategy'',  and show that it maximizes the root mean squared error (RMSE) amongst all single query strategies under the assumption that each sample in the dataset is drawn u.a.r. from $\{-1,1\}^{k+1}$, and the strategy is given knowledge of the empirical correlations of each of the first $k$ features with the $(k+1)$st feature (which can be obtained e.g. with $k$ non-adaptive queries asked prior to the adaptive query). We provide a pseudocode for the strategy 
\ifsupp 
in Algorithm~\ref{alg:strategy}, and prove that our single adaptive query results in maximum error, in Appendix~\ref{app:strategy}. To make the bounds comparable, we translate our worst-case confidence upper bounds for both the mechanisms to RMSE bounds in Theorem~\ref{thm:RZ_MSE} and Theorem~\ref{thm:thresh_mse}.
\fi
Figure~\ref{fig:two_round} shows the difference between our best upper bound and the realized RMSE (averaged over 100 executions) for the two mechanisms using $n=5,000$ and various values of $k$. (For the Gaussian, we set $\rho$ separately for each $k$, to minimize the upper bound.)  On the left, we see that the two bounds for the Gaussian mechanism are within a factor of 2.5, even for $k=50,000$ queries. Our bounds are thus reasonably tight in one important setting. For Thresholdout (right side), however, we see a large gap between the bounds which grows with $k$, even for our best query strategy\footnote{We tweak the adaptive query in the single-adaptive query strategy to result in maximum error for Thresholdout. We also tried ``tracing'' attack strategies (adapted from the fingerprinting lower bounds of \cite{BunUV14,HU14,SU15}) that contained multiple adaptive queries, but gave similar results.}. This result points to the promise for empirically-based confidence intervals for complex mechanisms that are harder to analyze.

\begin{figure}[ht]
	\centering	
	\begin{tabular}{cc}
	\hspace*{-10pt}	
	\begin{minipage}[b]{0.48\columnwidth}
	\includegraphics[width=\textwidth]{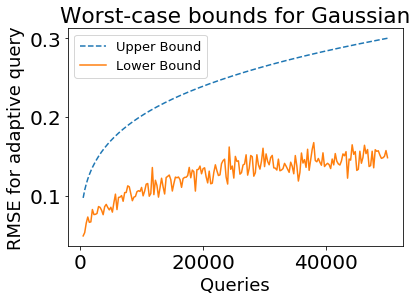}
	\end{minipage}
	\begin{minipage}[b]{0.5\columnwidth}
		\includegraphics[width=\textwidth]{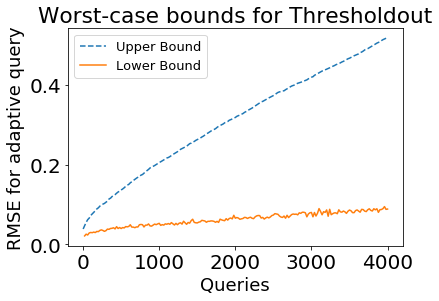}
	\end{minipage}
	\end{tabular}
	\caption{Worst-case upper (proven) and lower RMSE bounds (realized via the single-adaptive query strategy) with $n=5,000$ for Gaussian (left) and Thresholdout (right).}
	\label{fig:two_round}
\end{figure}


\section{The Guess and Check Framework}
\label{sec:gnc}
In light of the inadequacy of worst-case bounds, we here present our Guess and Check (GnC) framework which can go beyond the worst case. It takes as inputs \emph{guesses} for both the point estimate of a query, and a confidence interval width. If GnC can validate a guess, it releases the guess. Otherwise, at the cost of widening the confidence intervals provided for future guesses, it provides the guessed confidence width along with a point estimate for the query using the holdout set such that the guessed width is valid for the estimate.  

An instance of GnC, $\cM$, takes as input a dataset $X$, desired confidence level $1-\beta$, and a mechanism $\cM_g$ which operates on inputs of size $n_g < n$. $\cM$ randomly splits $X$ into two, giving one part $X_g$ to $\cM_g$, and reserving the rest as a holdout $X_h$. For each query $\phi_i$, mechanism $\cM_g$ uses $X_g$ to make a ``guess'' $(a_{g,i}, \tau_i)$ to $\cM$, for which $\cM$ conducts a validity check. If the check succeeds, then $\cM$ releases the guess as is, otherwise $\cM$ uses the holdout $X_h$ to provide a response containing a discretized answer that has $\tau_i$ as a valid confidence interval. This is closely related to Thresholdout. 
However, an important distinction is that the width of the target confidence interval, rather than just a point estimate, is provided as a guess. Moreover, the guesses themselves can be made by non-trivial algorithms. 

Depending on how long one expects/requires GnC to run, the input confidence parameter $\beta$ can guide the minimum value of the holdout size $n_h$ that will be required for GnC to be able to get a holdout width $\tau_h$ smaller than the desired confidence widths $\tau_i, \forall i \geq 1$. Note that this can be evaluated before starting GnC. Apart from that, we believe what is a good split will largely depend on the Guess mechanism. Hence, in general the split parameter should be treated as a hyperparameter for our GnC method.
We provide pseudocode for GnC in \Cref{alg:gnc}, and a block schematic of how a query is answered by GnC in \Cref{fig:gnc}.

\begin{figure}[tbp]
  \begin{algorithm}[H]
    \caption{Guess and Check}
\label{alg:gnc}
\begin{algorithmic}
\REQUIRE Data $X \in \cX^n$, confidence parameter $\beta$, analyst having mechanism $\cM_g$ with inputs of size $n_g<n$
\STATE Randomly split $X$ into a guess set $X_g$ of size $n_g$ to input into $\cM_g$, and a holdout $X_h$ of size $n_h = n - n_g$ 
\STATE $f \leftarrow 0$, $c_j \leftarrow \frac{6}{\pi^2 (j+1)^2}$ for $j \geq 0$ \comm{should just satisfy $\sum\limits_{j \geq 0}c_j \leq 1$}
\FOR{$i=1$ to $\infty$}
\STATE \textbf{if} $f > 0$ \textbf{then} $\nu_{i,f,\gamma_1^f} \leftarrow \binom{i-1}{f}  \prod_{j \in [f]} \left( \frac{1}{\gamma_j}\right)$ \textbf{else} $\nu_{i,f,\gamma_1^f} \leftarrow 1$ \hfill \comm{Compute \# possible transcripts}
\STATE $\beta_i \leftarrow (\beta \cdot  c_{i-1} \cdot  c_f)/\nu_{i,f,\gamma_1^f}$ 
\STATE Receive query $\phi_i$ and guess $\left(a_{g, i}, \tau_i\right) \leftarrow \cM_g(X_g, \phi_i)$ from analyst
\STATE  $a_{h, i} \leftarrow \phi_i(X_h)$ \hfill \comm{holdout answer}
\STATE  $\tau_h \leftarrow HoldoutTol(\beta_i, a_{g, i}, \tau_i, a_{h,
  i})$ \hfill \comm{$HoldoutTol$ returns a valid tolerance for $a_{h, i}$}
\IF{$|a_{g, i} - a_{h, i}| \leq \tau_i - \tau_h$}
\STATE Output $\left(a_{g, i}, \tau_i\right)$ to analyst
\ELSE
\STATE $f \leftarrow f + 1$
\STATE $\gamma_f \leftarrow \max\limits_{[0,\tau_i)} \gamma \text{
  s.t. } 2e^{-2(\tau_i - \gamma)^2n_h} \leq \beta_i$ \hfill \comm{max. discretization parameter with validity}
\IF{$\gamma_f > 0$}
\STATE Output $\left(\lfloor a_{h,i}\rfloor_{\gamma_f}, \tau_i\right)$ to analyst, where $\lfloor y\rfloor_{\gamma}$ denotes $y$ discretized to multiples of $\gamma$
\ELSE
\STATE Output $\bot$ to analyst
\STATE \textbf{break} \hfill  \comm{Terminate \textbf{for} loop}
\ENDIF
\ENDIF
\ENDFOR
\end{algorithmic}
  \end{algorithm}

\bigskip

\centering
  \includegraphics[width=\linewidth]{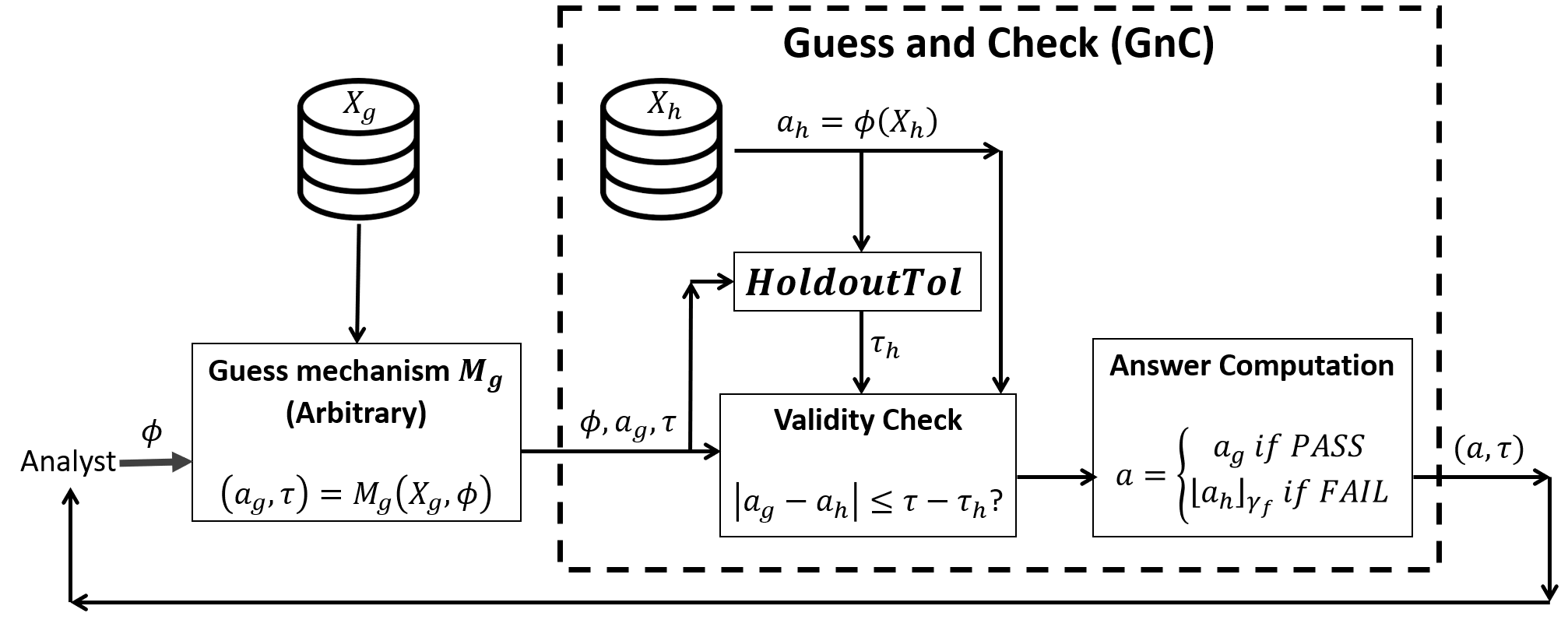}
  \caption{A schematic of how query $\phi$ is answered via our Guess and Check (GnC) framework. Dataset $X_g$ is the guess set randomly partitioned by GnC. The dotted box represents computations that are previleged, and are not accessible to the analyst.}
	\label{fig:gnc}
\end{figure}

We provide coverage guarantees for GnC without any restrictions on the guess mechanism. To get the guarantee, we first show that for query $\phi_i$, if function $HoldoutTol$ returns a  $(1 - \beta_i)$-confidence interval $\tau_h$ for holdout answer $a_{h,i}$, and GnC's output is the guess $(a_{g,i}, \tau_i)$, then $\tau_i$ is a  $(1 - \beta_i)$-confidence interval for $a_{g,i}$. We can get a simple definition for $HoldoutTol$  (formally stated in \ifsupp \Cref{app:proofGNCmain}), \fi but we provide a slightly sophisticated variant below that uses the guess and holdout answers to get better tolerances, especially under \emph{low-variance} queries. We defer the proof of \Cref{lem:gnc_width_mgf} to \ifsupp
Appendix~\ref{app:proofGNClem}.
\fi

\begin{lemma}\label{lem:gnc_width_mgf}
If the function $HoldoutTol$ in GnC (Algorithm~\ref{alg:gnc}) is defined as 
{\small
\begin{align*}
 HoldoutTol(\beta', a_g, \tau, a_h) 
 &  = 
	 \begin{cases}
            \arg\min\limits_{\tau' \in (0,\tau)} \left\{ \begin{array}{c}  \left(\frac{1 + \mu (e^{\ell} - 1)} {e^{\ell (\mu + \tau')}}\right)^n \leq \frac{\beta}{2}, \\ \text{ where } \ell \text{ solves } \\ \frac{\mu e^\ell}{1 + \mu(e^\ell + 1)} = \mu + \tau' \end{array} \right\}
             \text{if } a_g > a_h \, ,\\
		 \arg\min\limits_{\tau' \in (0,\tau)} \left\{ \begin{array}{c}  \left(\frac{ 1 + \mu' (e^{\ell} - 1) }{e^{\ell  (\mu' + \tau')}}\right)^n \leq \frac{\beta}{2}, \\ \text{ where }   \ell \text{ solves } \\  \frac{\mu' e^{\ell}}{1 + \mu'(e^{\ell} + 1)} = \mu' + \tau' \end{array} \right\}
                 \text{o.w.}
       \end{cases}
\end{align*}
}
where $\mu = a_g - \tau$ and $\mu' = 1 - a_g - \tau$, then for each query $\phi_i$ s.t. GnC's  output is $(a_{g, i}, \tau_i)$, we have $\Pr{\left(|a_{g, i} - \phi_i(\cD)|>\tau_i\right)}\leq \beta_i.$
\end{lemma}

Next, if failure $f$ occurs within GnC for query $\phi_i$, by applying a Chernoff bound we get that $\gamma_f$ is the maximum possible discretization parameter s.t. $\tau_i$ is a $(1 - \beta_i)$-confidence interval for the discretized holdout answer $\lfloor a_{h,i}\rfloor_{\gamma_f}$. 
Finally, we get a simultaneous coverage guarantee for GnC by a union bound over the error probabilities of the validity over all possible transcripts between GnC and any analyst $\adv$ with adaptive queries $\{\phi_1, \ldots, \phi_k\}$. The guarantee is stated below, with its proof deferred to  
\ifsupp
Appendix~\ref{app:proofGNCmain}.
\fi
\begin{theorem}\label{thm:gnc_acc}
The Guess and Check mechanism  (Algorithm~\ref{alg:gnc}), with inputs dataset $X\sim \cD^n$, confidence parameter $\beta$, and mechanism $\cM_g$ that, using inputs of size $n_g<n$, provides responses (``guesses'') of the form $(a_{g,i}, \tau_i)$ for query $\phi_i$, has simultaneous coverage $1 - \beta$. 
\end{theorem}


\subsection{Experimental evaluation}
\label{sec:expts}

Now, we provide details of our empirical evaluation of the Guess and Check framework. In our experiments, we use two mechanisms, namely the Gaussian mechanism and Thresholdout, for providing guesses in GnC. For brevity, we refer to the overall mechanism as GnC Gauss when the Gaussian is used to provide guesses, and GnC Thresh when Thresholdout is used. 

\mypar{Strategy for performance evaluation} Some  mechanisms evaluated in our experiments provide worst-case bounds, whereas the performance of others is instance-dependent and relies on the amount of adaptivity present in the querying strategy. To highlight the advantages of the latter, we design a query strategy called the quadratic-adaptive query strategy. 
Briefly, it contains two types of queries: random non-adaptive queries in which each sample's contribution is generated i.i.d. from a Bernoulli distribution, and adaptive queries which are linear combinations of previous queries. The adaptive queries become more sparsely distributed with time; ``hard'' adaptive queries $\phi_i, i>1,$ are asked when $i$ is a perfect square. They are computed in a similar manner as in the strategy used in \Cref{fig:two_round}. We provide pseudocode for the strategy in \ifsupp
Algorithm~\ref{alg:strategy}.
\fi 

\mypar{Experimental Setup} We run the quadratic-adaptive strategy for up to $40,000$ queries. We tune the hyperparameters of each mechanism to optimize for this query strategy. We fix a confidence parameter $\beta$ and set a target upper bound $\tau$ on the maximum allowable error we can tolerate, given our confidence bound. We evaluate each mechanism by the number of queries it can empirically answer with a confidence width of $\tau$ for our query strategy while providing a simultaneous coverage of $1-\beta$: i.e. the largest number of queries it can answer while providing $(\tau,\beta)$-accuracy. We plot the average and standard deviation of the number of queries $k$ answered before it exceeds its target error bound in 20 independent runs over the sampled data and the mechanism's randomness. When we plot the actual realized error for any mechanism, we denote it by dotted lines, whereas the provably valid error bounds resulting from the confidence intervals produced by GnC are denoted by solid lines. Note that the empirical error denoted by dotted lines is not actually possible to know without access to the distribution, and is plotted just to visualize the tightness of the provable confidence intervals. We compare to two simple baselines: sample splitting, and answer discretization: the better of these two is plotted as the thick solid line. For comparison, the best worst-case bounds for the Gaussian mechanism (\Cref{thm:RZ_CW}) are shown as dashed lines. Note that we improve by roughly two orders of magnitude compared to the tightest bounds for the Gaussian. We improve over the baseline at dataset sizes $n \geq 2,000$.

\mypar{Boost in performance for low-variance queries}  Since all the queries we construct take binary values on a sample $x \in \cX$, the variance of query $\phi_i$ is given by $var(\phi_i)=\phi_i(\cD)(1-\phi_i(\cD))$, as $\phi_i(\cD)=\Pr{(\phi_i(x)=1)}$. Now, $var(\phi_i)$ is maximized when $\phi_i(\cD)=0.5$. Hence, informally, we denote query  $\phi_i$ as low-variance if either $\phi_i(\cD) \ll 0.5$, or $\phi_i(\cD) \gg 0.5$. We want to be able to adaptively provide tighter confidence intervals for low-variance queries (as,  for e.g., the worst-case bounds of \cite{FS17,FS17b} are able to). For instance, in \Cref{fig:gnc_main1} (left), we show that in the presence of low-variance queries, using \Cref{lem:gnc_width_mgf} for $HoldoutTol$ (plot labelled ``GnC Check:MGF'') results in a significantly better performance for GnC Gauss as compared to using \ifsupp \Cref{lem:gnc_width1} \fi (plot labelled ``GnC Check:Chern''). We fix $\tau,\beta=0.05$, and set $\phi_i(\cD)=0.9$ for $i\geq 1$. We can see that as the dataset size grows, using \Cref{lem:gnc_width_mgf} provides an improvement of almost 2 orders of magnitude in terms of the number of queries $k$ answered. This is due to  \Cref{lem:gnc_width_mgf} providing tighter holdout tolerances $\tau_h$ for low-variance queries (with guesses close to $0$ or $1$), compared to those obtained via \ifsupp \Cref{lem:gnc_width1} \fi (agnostic to the query variance). Thus, we use \Cref{lem:gnc_width_mgf} for $HoldoutTol$ in all experiments with GnC below. The worst-case bounds for the Gaussian don't promise a coverage of $1 - \beta$ even for $k=1$ in the considered parameter ranges. This is representative of a general phenomenon: switching to GnC-based bounds instead of worst-case bounds is often the difference between obtaining useful vs. vacuous guarantees.

\begin{figure}[ht]
	\centering
	\begin{tabular}{cc}
	\hspace*{-10pt}	
	\begin{minipage}[b]{0.5\columnwidth}
		\includegraphics[width=\textwidth]{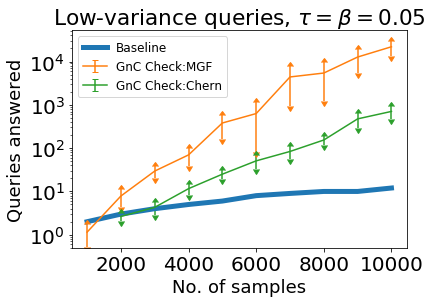}
	\end{minipage}
	\begin{minipage}[b]{0.5\columnwidth}
		\includegraphics[width=\textwidth]{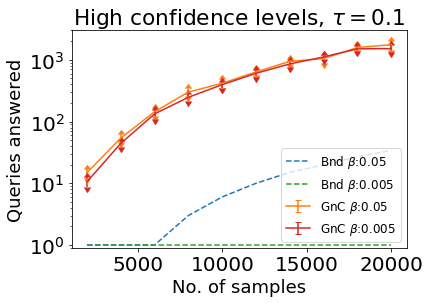}
	\end{minipage}
\end{tabular}
	\caption{ \emph{Left:} Gain in performance for GnC Gauss by using \Cref{lem:gnc_width_mgf} for $HoldoutTol$ (``GnC Check:MGF''), as compared to \ifsupp using \Cref{lem:gnc_width1} \fi (``GnC Check:Chern''). \emph{Right:} Performance of GnC Gauss (``GnC''), and the best Gaussian bounds (``Bnd''), for $\beta\in \{0.05, 0.005\}$.}
	\label{fig:gnc_main1}
\end{figure}

\mypar{Performance at high confidence levels} The bounds we prove for the Gaussian mechanism, which are the best known worst-case bounds for the considered sample size regime, have a substantially sub-optimal dependence on the coverage parameter $\beta:\sqrt{1/\beta}$.  On the other hand, sample splitting (and the bounds from \cite{DFHPRR15STOC,BNSSSU16} which are asymptotically optimal but vacuous at small sample sizes) have a much better dependence on $\beta: \ln{\left(1/2\beta\right)}$. Since the coverage bounds of GnC are strategy-dependent, the dependence of $\tau$ on $\beta$ is not clear a priori. In \Cref{fig:gnc_main1} (right), we show the performance of GnC Gauss (labelled ``GnC'') when $\beta \in \{0.05, 0.005\}$. We see that reducing $\beta$ by a factor of 10 has a negligible effect on GnC's performance. Note that this is the case even though the guesses are provided by the Gaussian, for which we do not have non-vacuous bounds with a mild dependence on $\beta$ in the considered parameter range (see the worst-case bounds, plotted as ``Bnd'') --- even though we might conjecture that such bounds exist. This gives an illustration of how GnC can correct deficiencies in our worst-case theory: conjectured improvements to the theory can be made rigorous with GnC's certified confidence intervals. 

\mypar{Guess and Check with different guess mechanisms} GnC is designed to be modular, enabling it to take advantage of arbitrarily complex mechanisms to make guesses. Here, we compare the performance of two such mechanisms for making guesses, namely the Gaussian mechanism, and Thresholdout. In \Cref{fig:gnc_add1} (left), we first plot the number of queries answered by the Gaussian (``Gauss Emp'') and Thresholdout (``Thresh Emp'') mechanisms, respectively, until the maximum empirical error of the query answers exceeds $\tau = 0.1$. It is evident that Thresholdout, which uses an internal holdout set to answer queries that likely overfit to its training set, provides better performance than the Gaussian mechanism.  In fact, we see that for $n>5000$, while Thresholdout is always able to answer $40,000$ queries (the maximum number of queries we tried in our experiments), the Gaussian mechanism isn't able to do so even for the largest  dataset size we consider. Note that the ``empirical'' plots are generally un-knowable in practice, since we do not have access to the underlying distributions. But they serve as upper bounds for the best performance a mechanism can provide.

Next, we fix $\beta=0.05$, and plot the performance of GnC Gauss and GnC Thresh. We see that even though GnC Thresh has noticeably higher variance, it provides performance that is close to two orders of magnitude larger than GnC Gauss when $n\geq8000$. Moreover, for $n\geq 8000$, it is interesting to see GnC Thresh guarantees $(\tau,\beta)$-accuracy for our strategy while consistently beating even the empirical performance of the Gaussian. We note that the best bounds for both the Gaussian and Thresholdout mechanisms alone (not used as part of GnC) do not provide any non-trivial guarantees in the considered parameter ranges. 

\begin{figure}[ht]
	\centering
	\begin{tabular}{cc}
	\hspace*{-10pt}
	\begin{minipage}[b]{0.49\columnwidth}
		\includegraphics[width=\textwidth]{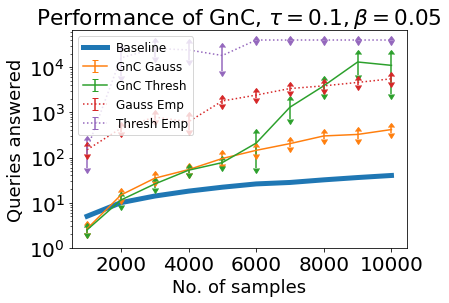}
	\end{minipage}
	\begin{minipage}[b]{0.5\columnwidth}
		\includegraphics[width=\textwidth]{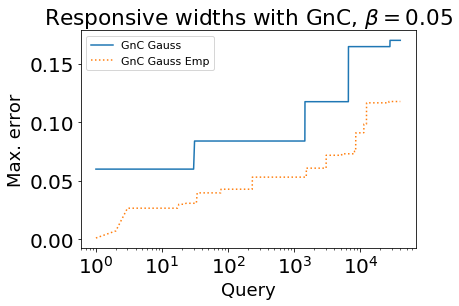}
	\end{minipage}
\end{tabular}
	\caption{ \emph{Left:} Performance of GnC with Gaussian (``GnC Gauss''), and Thresholdout (``GnC Thresh'') mechanisms, together with their empirical error. \emph{Right:} The accuracy of GnC with Gaussian guesses  provides ``responsive'' confidence interval widths that closely track the empirical error incurred by the guesses of the Gaussian mechanism (``GnC Gauss Emp'').}
	\label{fig:gnc_add1}
\end{figure}

\mypar{Responsive widths that track the empirical error} The GnC framework is designed to certify guesses which represent both a point estimate and a desired confidence interval width for each query. Rather than having fixed confidence interval widths, this framework also provides the flexibility to incorporate guess mechanisms that provide increased interval  widths as failures accumulate within GnC. This allows GnC to be able to re-use the holdout set in perpetuity, and answer an infinite number of queries (albeit with confidence widths that might grow to be vacuous).

In \Cref{fig:gnc_add1} (right), we fix $n=30000, \beta=0.05, \tau_1 = 0.06$, and plot the performance of GnC Gauss such that the guessed confidence width $\tau_{i+1} = \min{(1.4\tau_{i}, 0.17)}$ if the ``check'' for query $\phi_i$ results in a failure, otherwise $\tau_{i+1} = \tau_i$. For comparison, we also plot the actual maximum empirical error encountered by the answers provided by GnC (``GnC Gauss Emp''). It corresponds to the maximum empirical error of the answers of the Gaussian mechanism that is used as a guess mechanism within GnC, unless the check for a query results in a failure (which occurs 4 times in 40000 queries), in which case the error corresponds to the discretized answer on the holdout. We see that the statistically valid accuracy guaranteed by GnC is ``responsive'' to the empirical error of the realized answers produced by the GnC, and is almost always within a factor of 2 of the actual error.

\mypar{Discussion} The runtime of our GnC system is dominated by the runtime of the mechanism providing the guesses. For each guess, the GnC system need only compute the empirical  answer of the query on the holdout set, and a width (for example, from \Cref{lem:gnc_width_mgf}) that comes  from a simple one-dimensional optimization. Thus, GnC with any particular Guess mechanism will have an execution time comparable to that of the Guess mechanism by itself. It is also important to note that GnC can be combined with any guess-generating mechanism, and it will inherit the worst-case generalization behavior of that mechanism. However, the GnC will typically provide much tighter confidence bounds (since the worst-case bounds are typically loose).


\section{Conclusion}

In this work, we focus on algorithms that provide
explicit confidence
intervals with \emph{sound} coverage probabilities for adaptively posed statistical queries. We start by deriving tighter worst-case bounds for several mechanisms, and show that our improved bounds are within small constant factors of optimal for certain mechanisms. Our main contribution is the Guess and Check framework, that allows an analyst to use \emph{any method} for ``guessing'' point estimates and confidence interval widths for their adaptive queries, and then rigorously validate those guesses on an additional held-out dataset. Our empirical evaluation demonstrates that GnC can improve on worst-case bounds by orders of magnitude, and that it improves on the naive baseline even for  modest sample sizes. We also provide a Python library (\cite{git_url})
implementing our GnC method.


\section*{Acknowledgements}

The authors would like to thank Omer Tamuz for helpful comments regarding a conjecture that existed in a prior version of this work. A.R. acknowledges support in part by a grant from the Sloan Foundation, and NSF grants AF-1763314 and CNS-1253345. A.S. and O.T. were supported in part by a grant from the Sloan foundation, and NSF grants IIS-1832766 and AF-1763786.  B.W. is supported by the NSF GRFP (award No. 1754881).  This work was done in part while R.R., A.S., and O.T. were visiting the Simons Institute for the Theory of Computing. 

\bibliographystyle{plainnat}
\bibliography{refs}

\begin{thebibliography}{29}
\providecommand{\natexlab}[1]{#1}
\providecommand{\url}[1]{\texttt{#1}}
\expandafter\ifx\csname urlstyle\endcsname\relax
  \providecommand{\doi}[1]{doi: #1}\else
  \providecommand{\doi}{doi: \begingroup \urlstyle{rm}\Url}\fi

\bibitem[Bassily et~al.(2016)Bassily, Nissim, Smith, Steinke, Stemmer, and
  Ullman]{BNSSSU16}
Raef Bassily, Kobbi Nissim, Adam Smith, Thomas Steinke, Uri Stemmer, and
  Jonathan Ullman.
\newblock Algorithmic stability for adaptive data analysis.
\newblock In \emph{Proceedings of the 48th Annual ACM SIGACT Symposium on
  Theory of Computing}, pages 1046--1059. ACM, 2016.

\bibitem[Bun and Steinke(2016{\natexlab{a}})]{BS16}
Mark Bun and Thomas Steinke.
\newblock Concentrated differential privacy: Simplifications, extensions, and
  lower bounds.
\newblock In Martin Hirt and Adam Smith, editors, \emph{Theory of
  Cryptography}, pages 635--658, Berlin, Heidelberg, 2016{\natexlab{a}}.
  Springer Berlin Heidelberg.
\newblock ISBN 978-3-662-53641-4.

\bibitem[Bun and Steinke(2016{\natexlab{b}})]{BunS16}
Mark Bun and Thomas Steinke.
\newblock Concentrated differential privacy: Simplifications, extensions, and
  lower bounds.
\newblock \emph{CoRR}, abs/1605.02065, 2016{\natexlab{b}}.
\newblock URL \url{http://arxiv.org/abs/1605.02065}.

\bibitem[Bun et~al.(2014)Bun, Ullman, and Vadhan]{BunUV14}
Mark Bun, Jonathan Ullman, and Salil~P. Vadhan.
\newblock Fingerprinting codes and the price of approximate differential
  privacy.
\newblock In \emph{STOC}, pages 1--10. ACM, May 31 -- June 3 2014.

\bibitem[Dwork and Rothblum(2016)]{DR16}
Cynthia Dwork and Guy~N. Rothblum.
\newblock Concentrated differential privacy.
\newblock \emph{CoRR}, abs/1603.01887, 2016.

\bibitem[Dwork et~al.(2006{\natexlab{a}})Dwork, Kenthapadi, McSherry, Mironov,
  and Naor]{DKMMN06}
Cynthia Dwork, Krishnaram Kenthapadi, Frank McSherry, Ilya Mironov, and Moni
  Naor.
\newblock Our data, ourselves: Privacy via distributed noise generation.
\newblock In \emph{Advances in Cryptology - {EUROCRYPT} 2006, 25th Annual
  International Conference on the Theory and Applications of Cryptographic
  Techniques, St. Petersburg, Russia, May 28 - June 1, 2006, Proceedings},
  pages 486--503, 2006{\natexlab{a}}.
\newblock \doi{10.1007/11761679_29}.

\bibitem[Dwork et~al.(2006{\natexlab{b}})Dwork, McSherry, Nissim, and
  Smith]{DMNS06}
Cynthia Dwork, Frank McSherry, Kobbi Nissim, and Adam Smith.
\newblock Calibrating noise to sensitivity in private data analysis.
\newblock In \emph{Theory of Cryptography Conference}, pages 265--284.
  Springer, 2006{\natexlab{b}}.

\bibitem[Dwork et~al.(2010)Dwork, Rothblum, and Vadhan]{DRV10}
Cynthia Dwork, Guy~N. Rothblum, and Salil~P. Vadhan.
\newblock Boosting and differential privacy.
\newblock In \emph{51th Annual {IEEE} Symposium on Foundations of Computer
  Science, {FOCS} 2010, October 23-26, 2010, Las Vegas, Nevada, {USA}}, pages
  51--60, 2010.
\newblock \doi{10.1109/FOCS.2010.12}.

\bibitem[Dwork et~al.(2015{\natexlab{a}})Dwork, Feldman, Hardt, Pitassi,
  Reingold, and Roth]{DFHPRR15nips}
Cynthia Dwork, Vitaly Feldman, Moritz Hardt, Toni Pitassi, Omer Reingold, and
  Aaron Roth.
\newblock Generalization in adaptive data analysis and holdout reuse.
\newblock In \emph{Advances in Neural Information Processing Systems}, pages
  2350--2358, 2015{\natexlab{a}}.

\bibitem[Dwork et~al.(2015{\natexlab{b}})Dwork, Feldman, Hardt, Pitassi,
  Reingold, and Roth]{DFHPRR15Science}
Cynthia Dwork, Vitaly Feldman, Moritz Hardt, Toniann Pitassi, Omer Reingold,
  and Aaron Roth.
\newblock The reusable holdout: Preserving validity in adaptive data analysis.
\newblock \emph{Science}, 349\penalty0 (6248):\penalty0 636--638,
  2015{\natexlab{b}}.
\newblock \doi{10.1126/science.aaa9375}.
\newblock URL \url{http://www.sciencemag.org/content/349/6248/636.abstract}.

\bibitem[Dwork et~al.(2015{\natexlab{c}})Dwork, Feldman, Hardt, Pitassi,
  Reingold, and Roth]{DFHPRR15}
Cynthia Dwork, Vitaly Feldman, Moritz Hardt, Toniann Pitassi, Omer Reingold,
  and Aaron~Leon Roth.
\newblock Preserving statistical validity in adaptive data analysis.
\newblock In \emph{Proceedings of the Forty-Seventh Annual ACM on Symposium on
  Theory of Computing}, pages 117--126. ACM, 2015{\natexlab{c}}.

\bibitem[Dwork et~al.(2015{\natexlab{d}})Dwork, Feldman, Hardt, Pitassi,
  Reingold, and Roth]{DFHPRR15STOC}
Cynthia Dwork, Vitaly Feldman, Moritz Hardt, Toniann Pitassi, Omer Reingold,
  and Aaron~Leon Roth.
\newblock Preserving statistical validity in adaptive data analysis.
\newblock In \emph{Proceedings of the Forty-Seventh Annual ACM on Symposium on
  Theory of Computing}, STOC '15, pages 117--126, New York, NY, USA,
  2015{\natexlab{d}}. ACM.
\newblock ISBN 978-1-4503-3536-2.
\newblock \doi{10.1145/2746539.2746580}.

\bibitem[Feldman and Steinke(2017{\natexlab{a}})]{FS17}
Vitaly Feldman and Thomas Steinke.
\newblock Generalization for adaptively-chosen estimators via stable median.
\newblock In \emph{Proceedings of the 30th Conference on Learning Theory,
  {COLT} 2017, Amsterdam, The Netherlands, 7-10 July 2017}, pages 728--757,
  2017{\natexlab{a}}.
\newblock URL \url{http://proceedings.mlr.press/v65/feldman17a.html}.

\bibitem[Feldman and Steinke(2017{\natexlab{b}})]{FS17b}
Vitaly Feldman and Thomas Steinke.
\newblock Calibrating noise to variance in adaptive data analysis.
\newblock \emph{CoRR}, abs/1712.07196, 2017{\natexlab{b}}.
\newblock URL \url{http://arxiv.org/abs/1712.07196}.

\bibitem[Gelman and Loken(2014)]{GL14}
Andrew Gelman and Eric Loken.
\newblock The statistical crisis in science.
\newblock \emph{American Scientist}, 102\penalty0 (6):\penalty0 460, 2014.

\bibitem[Gossmann et~al.(2018)Gossmann, Pezeshk, and Sahiner]{GPS18}
Alexej Gossmann, Aria Pezeshk, and Berkman Sahiner.
\newblock Test data reuse for evaluation of adaptive machine learning
  algorithms: over-fitting to a fixed'test'dataset and a potential solution.
\newblock In \emph{Medical Imaging 2018: Image Perception, Observer
  Performance, and Technology Assessment}, volume 10577, page 105770K.
  International Society for Optics and Photonics, 2018.

\bibitem[Gray(1990)]{Gray90}
Robert~M. Gray.
\newblock \emph{Entropy and Information Theory}.
\newblock Springer-Verlag, Berlin, Heidelberg, 1990.
\newblock ISBN 0-387-97371-0.

\bibitem[Hardt and Ullman(2014)]{HU14}
Moritz Hardt and Jonathan Ullman.
\newblock Preventing false discovery in interactive data analysis is hard.
\newblock In \emph{Foundations of Computer Science (FOCS), 2014 IEEE 55th
  Annual Symposium on}, pages 454--463. IEEE, 2014.

\bibitem[Jung et~al.(2019)Jung, Ligett, Neel, Roth, Sharifi-Malvajerdi, and
  Shenfeld]{JungLNRSS19}
Christopher Jung, Katrina Ligett, Seth Neel, Aaron Roth, Saeed
  Sharifi-Malvajerdi, and Moshe Shenfeld.
\newblock A new analysis of differential privacy's generalization guarantees,
  2019.

\bibitem[Kairouz et~al.(2017)Kairouz, Oh, and Viswanath]{KOV15}
Peter Kairouz, Sewoong Oh, and Pramod Viswanath.
\newblock The composition theorem for differential privacy.
\newblock \emph{{IEEE} Trans. Information Theory}, 63\penalty0 (6):\penalty0
  4037--4049, 2017.

\bibitem[{Kasiviswanathan} and {Smith}(2014)]{KS14}
S.P. {Kasiviswanathan} and A.~{Smith}.
\newblock {On the `Semantics' of Differential Privacy: A Bayesian Formulation}.
\newblock \emph{Journal of Privacy and Confidentiality}, Vol. 6: Iss. 1,
  Article 1, 2014.

\bibitem[Mania et~al.(2019)Mania, Miller, Schmidt, Hardt, and
  Recht]{modelsimilarity}
Horia Mania, John Miller, Ludwig Schmidt, Moritz Hardt, and Benjamin Recht.
\newblock Model similarity mitigates test set overuse.
\newblock \emph{arXiv preprint arXiv:1905.12580}, 2019.

\bibitem[Rogers et~al.(2016)Rogers, Roth, Smith, and Thakkar]{RRST16}
Ryan Rogers, Aaron Roth, Adam Smith, and Om~Thakkar.
\newblock Max-information, differential privacy, and post-selection hypothesis
  testing.
\newblock In \emph{Foundations of Computer Science (FOCS), 2016 IEEE 57th
  Annual Symposium on}, 2016.

\bibitem[Rogers et~al.(2019)Rogers, Roth, Smith, Srebro, Thakkar, and
  Woodworth]{git_url}
Ryan Rogers, Aaron Roth, Adam Smith, Nathan Srebro, Om~Thakkar, and Blake
  Woodworth.
\newblock Repository for empirical adaptive data analysis.
\newblock \url{https://github.com/omthkkr/empirical_adaptive_data_analysis},
  2019.

\bibitem[{Russo} and {Zou}(2015)]{RZ16arxiv}
D.~{Russo} and J.~{Zou}.
\newblock {How much does your data exploration overfit? Controlling bias via
  information usage}.
\newblock \emph{ArXiv e-prints}, November 2015.

\bibitem[Russo and Zou(2016)]{RZ16}
Daniel Russo and James Zou.
\newblock Controlling bias in adaptive data analysis using information theory.
\newblock In \emph{Proceedings of the 19th International Conference on
  Artificial Intelligence and Statistics, {AISTATS}}, 2016.

\bibitem[Steinke and Ullman(2015)]{SU15}
Thomas Steinke and Jonathan Ullman.
\newblock Interactive fingerprinting codes and the hardness of preventing false
  discovery.
\newblock In \emph{Proceedings of The 28th Conference on Learning Theory},
  pages 1588--1628, 2015.

\bibitem[Xu and Raginsky(2017)]{XR17}
Aolin Xu and Maxim Raginsky.
\newblock Information-theoretic analysis of generalization capability of
  learning algorithms.
\newblock In \emph{NIPS 2017, 4-9 December 2017, Long Beach, CA, {USA}}, pages
  2521--2530, 2017.

\bibitem[Zrnic and Hardt(2019)]{naturalanalyst}
Tijana Zrnic and Moritz Hardt.
\newblock Natural analysts in adaptive data analysis.
\newblock \emph{arXiv preprint arXiv:1901.11143}, 2019.

\end{thebibliography}

\ifsupp
\appendix

\section{Additional Preliminaries}
\label{app:defns}
Here, we present some additional preliminaries that were omitted from the main body.

\subsection{Confidence Interval Preliminaries}

In our implementation, we are comparing the true average $\phi(\cD)$ to the answer $a$, which will be the true answer on the sample with additional noise to ensure each query is stably answered.  
We then use the following string of inequalities to find the width $\tol$ of the confidence interval.
{
\begin{align*}
\prob{|\phi(\cD) - a|  \geq \tol}& \leq \prob{|\phi(\cD) - \phi(X)| + |\phi(X) - a| \geq \tol} \\
& \leq \underbrace{\prob{| \phi(\cD) - \phi(X)| \geq \tol - \tol'}}_{\text{Population Accuracy}}  + \underbrace{\prob{| \phi(X) - a| \geq  \tol'}}_{\text{Sample Accuracy}}  \text{ for } \tol' \geq 0\label{eq:errors} \numberthis
\end{align*}}

We will then use this connection to get a bound in terms of the accuracy on the sample and the error in the empirical average to the true mean.  Many of the results in this line of work use a \emph{transfer theorem} which states that if a query is selected via a private method, then the query evaluated on the sample is close to the true population answer, thus providing a bound on \emph{population accuracy}.   However, we also need to control the \emph{sample accuracy} which is affected by the amount of noise that is added to ensure stability.  We then seek a balance between the two terms, where too much noise will give terrible sample accuracy but great accuracy on the population -- due to the noise making the choice of query essentially independent of the data -- and too little noise makes for great sample accuracy but bad accuracy to the population.  We will consider Gaussian noise, and use the composition theorems to determine the scale of noise to add to achieve a target accuracy after $k$ adaptively selected statistical queries.  

Given the size of our dataset $n$, number of adaptively chosen statistical queries $k$, and confidence level $1-\beta$, we want to find what \emph{confidence width} $\tol$ ensures $\alg = (\alg_1,\cdots, \alg_k)$ is $(\tol,\beta)$-accurate with respect to the population when each algorithm $\alg_i$ adds either Laplace or Gaussian noise to the answers computed on the sample with some yet to be determined variance.  To bound the sample accuracy, we can use the following theorem that gives the accuracy guarantees of the Gaussian mechanism. 

\begin{theorem}
If $\{Z_i: i \in [k]\}\stackrel{i.i.d.}{\sim} N(0,\sigma^2)$ then for $\beta \in (0,1]$ we have: 
{
\begin{align*}
&  \prob{|Z_i| \geq  \sigma \sqrt{2\ln(2/\beta)}} \leq \beta  \implies \prob{\exists i \in [k] \text{ s.t. } |Z_i| \geq  \sigma \sqrt{2\ln(2k/\beta)}} \leq \beta.
\label{eq:acc_gauss} \numberthis
\end{align*}}
\label{thm:acc}
\end{theorem}

\subsection{Stability Measures}
It turns out that privacy preserving algorithms give strong stability guarantees which allows for the rich theory of differential privacy to extend to adaptive data analysis \citep{DFHPRR15STOC,DFHPRR15nips,BNSSSU16,RRST16}. In order to define these privacy notions, we define two datasets $X= (x_1,\cdots, x_n), X'=(x_1',\cdots, x_n') \in \cX^n$ to be \emph{neighboring} if they differ in at most one entry, i.e. there is some $i \in [n]$ where $x_i \neq x_i'$, but $x_j = x_j'$ for all $j \neq i$.  We first define \emph{differential privacy}.
\begin{definition}[Differential Privacy \citep{DMNS06,DKMMN06}]\label{defn:DP}
A randomized algorithm (or mechanism) $\alg: \cX^n \to \cY$ is $(\epsilon,\delta)$-differentially private (DP) if for all neighboring datasets $X$ and $X'$ and each outcome $S \subseteq \cY$, we have 
$
\prob{\alg(X) \in S} \leq e^{\epsilon}\prob{\alg(X') \in S} + \delta.
$
If $\delta = 0$, we simply say $\alg$ is $\epsilon$-DP or pure DP.  Otherwise for $\delta>0$, we say \emph{approximate} DP.
\end{definition}

We then give a more recent notion of privacy, called concentrated differential privacy (CDP), which can be thought of as being ``in between" pure and approximate DP.  In order to define CDP, we define the privacy loss random variable which quantifies how much the output distributions of an algorithm on two neighboring datasets can differ.  
\begin{definition}[Privacy Loss]
Let $\alg: \cX^n \to \cY$ be a randomized algorithm.  For neighboring datasets $X,X' \in \cX^n$, let $Z(y) = \ln\left( \frac{\prob{\alg(X) = y} }{\prob{\alg(X') = y} } \right)$. We then define the privacy loss variable $\PL{\alg(X)}{\alg(X')}$ to have the same distribution as $Z(\alg(X))$.  
\label{defn:priv_loss}
\end{definition}
Note that if we can bound the privacy loss random variable with certainty over all neighboring datasets, then the algorithm is pure DP.  Otherwise, if we can bound the privacy loss with high probability then it is approximate DP (see \cite{KS14} for a more detailed discussion on this connection).  

We can now define \emph{zero concentrated differential privacy} (zCDP), given by \cite{BS16} (Note that \cite{DR16} initially gave a definition of CDP which \cite{BS16} then modified).  

\begin{definition}[zCDP]
An algorithm $\alg: \cX^n \to \cY$ is $\rho$-zero concentrated differentially private (zCDP), if for all neighboring datasets $X,X' \in \cX^n$  and all $\lambda >0$ we have 
$$
\ex{\exp\left( \lambda \left(\PL{\alg(X)}{\alg(X')} - \rho\right)\right)} \leq e^{\lambda^2 \rho}.
$$
\label{defn:zCDP}
\end{definition}

We then give the Laplace and Gaussian mechanism for statistical queries.
\begin{theorem}
Let $\phi: \cX \to [0,1]$ be a statistical query and $X \in \cX^n$.  The Laplace mechanism $\ML: \cX^n \to \R$ is the following 
$
\ML(X) = \frac{1}{n}\sum_{i = 1}^n\phi(x_i) + \Lap\left( \frac{1}{\epsilon n}\right)
$,
which is $\epsilon$-DP.  Further, the Gaussian mechanism $\MG: \cX^n \to \R$ is the following 
$
\MG(X) = \frac{1}{n}\sum_{i = 1}^n\phi(x_i) +N\left(0, \frac{1}{2\rho n^2}\right)
$,
which is $\rho$-zCDP.
\end{theorem}

We now give the advanced composition theorem for $k$-fold adaptive composition.
\begin{theorem}[\cite{DRV10},\cite{KOV15}]
The class of $\epsilon'$-DP algorithms is $(\epsilon,\delta)$-DP under $k$-fold adaptive composition where $\delta>0$ and
\begin{equation}
\epsilon = \left(\frac{e^{\epsilon'}-1}{e^{\epsilon'}+1}\right) \epsilon' k + \epsilon'\sqrt{2k \ln(1/\delta)}
\label{eq:adv_comp}
\end{equation}
\label{thm:drv}
\end{theorem}

We will also use the following results from zCDP.
\begin{theorem}[\cite{BS16}]
The class of $\rho$-zCDP algorithms is $k\rho$-zCDP under $k$-fold adaptive composition.  Further if $\cM$ is $\epsilon$-DP then $\cM$ is $\epsilon^2/2$-zCDP and if $\cM$ is $\rho$-zCDP then $\cM$ is $(\rho + 2\sqrt{\rho\ln(\sqrt{\pi \rho}/\delta)},\delta)$-DP for any $\delta>0$.
\label{thm:zCDP}
\end{theorem}

Another notion of stability that we will use is mutual information (in nats) between two random variables: the input $X$ and output $\alg(X)$.
\begin{definition}[Mutual Information]
Consider two random variables $X$ and $Y$ and let $Z(x,y) = \ln \left( \frac{\prob{(X,Y) = (x,y) } }{\prob{X = x} \prob{Y = y} } \right)$.  We then denote the mutual information as $\MutInfo{X}{Y} = \Ex{}{Z(X,Y)}$, where the expectation is taken over the joint distribution of $(X,Y)$.
\label{defn:mutual_info}
\end{definition}

\subsection{Monitor Argument}
For the population accuracy term in \eqref{eq:errors}, we will use the \emph{monitor argument} from \cite{BNSSSU16}.  Roughly, this analysis allows us to obtain a bound on the population accuracy over $k$ rounds of interaction between adversary $\adv$ and algorithm $\alg$ by only considering the difference $|\phi(X) - \phi(\cD)|$ for the two stage interaction where $\phi$ is chosen by $\adv$ based on outcome $\alg(X)$.  We present the monitor $\cW_\cD[\alg,\adv]$ in \Cref{alg:monitor}.

\begin{algorithm}
\caption{Monitor $\cW_\cD[\alg,\adv](X)$}
\label{alg:monitor}
\begin{algorithmic}
\REQUIRE $X \in \cX^n$
\STATE We simulate $\alg(X)$ and $\adv$ interacting.  We write $\phi_{1}, \cdots, \phi_{k} \in \cQ_{SQ}$ as the queries chosen by $\adv$ and write $a_{1},\cdots, a_{k} \in \R$ as the corresponding answers of $\alg$.
\STATE Let 
$
j^* = \myargmax_{j \in [k]} \left| \phi_{j}(\cD) - a_{j} \right|.
$
\ENSURE $\phi_{j^*}$
\end{algorithmic}
\end{algorithm}

Since our stability definitions are closed under post-processing, we can substitute the monitor $\cW_\cD[\alg,\cA]$ as our post-processing function $f$ in the above theorem.  We then get the following result. 

\begin{corollary}
Let $\alg = (\alg_1,\cdots, \alg_k)$, where each $\alg_i$ may be adaptively chosen, satisfy any stability condition that is closed under post-processing.  For each $i \in [k]$, let $\phi_i$ be the statistical query chosen by adversary $\adv$ based on answers $a_j = \alg_j(X), \forall j \in [i-1]$, and let $\phi$ be any function of $(a_1, \cdots, a_k)$.  Then, we have for $\tau' \geq 0$
{ 
\begin{align*}
 & \Prob{\substack{\alg,\\X \sim \cD^n}}{\max_{i \in [k]} |\phi_i(\cD) - a_i| \geq \tau}  \leq \Prob{\substack{\alg,\\X \sim \cD^n}}{\max_{i \in [k]} |\phi_i(X) - a_i| \geq \tau'}    +  \Prob{\substack{X \sim \cD^n,\\ \phi \gets \alg(X)}}{|\phi(\cD) - \phi(X)| \geq \tau - \tau'} 
\end{align*}
}
\label{cor:pop_acc}
\end{corollary}

\begin{proof}
From the monitor in \Cref{alg:monitor} and the fact that $\alg$ is closed under post-processing, we have
{ 
\begin{align*}
\Prob{\substack{\alg,\\X \sim \cD^n}}{\max_{i \in [k]} |\phi_i(\cD) - a_i| \geq \tau} &  =\Prob{\substack{X \sim \cD^n,\\ \phi_{j^*} \gets \cW_\cD[\alg,\adv](X)}}{ |\phi_{j^*}(\cD) - a_{j^*}| \geq \tau} \\
& \leq \Prob{\substack{X \sim \cD^n,\\ \phi_{j^*} \gets \cW_\cD[\alg,\adv](X)}}{ |\phi_{j^*}(\cD) - \phi_{j^*}(X)|| \geq \tau - \tau'}  \\
& \qquad + \Prob{\substack{X \sim \cD^n,\\ \phi_{j^*} \gets \cW_\cD[\alg,\adv](X)}}{ |\phi_{j^*}(X) - a_{j^*}|| \geq \tau'} \\
& \leq \Prob{\substack{X \sim \cD^n,\\ \phi \gets \alg(X)}}{|\phi(\cD) - \phi(X)| \geq \tau - \tau'}  + \Prob{\substack{\alg,\\X \sim \cD^n}}{\max_{i \in [k]} |\phi_i(X) - a_i| \geq \tau'} 
\end{align*}}
\end{proof}

We can then use the above corollary to obtain an accuracy guarantee by union bounding over the sample accuracy for all $k$ rounds of interaction and then bounding the population error for a single adaptively chosen statistical query.

\section{Confidence Interval Bounds from Prior Work}

\subsection{Confidence Bounds from \texorpdfstring{\cite{DFHPRR15nips}}{}}
\label{sec:dfh_res}
We start by deriving confidence bounds using results from \cite{DFHPRR15nips}, which uses the following transfer theorem (see Theorem 10 in \cite{DFHPRR15nips}).
\begin{theorem}
If $\alg$ is $(\epsilon,\delta)$-DP where $\phi \gets \alg(X)$ and $\tol \geq \sqrt{\frac{48}{n} \ln(4/\beta)}$, $\epsilon \leq \frac{\tol}{4}$, and $\delta = \exp\left(\frac{-4 \ln(8/\beta)}{\tol} \right)$, then $\prob{|\phi(\cD) - \phi(X) |\geq \tol} \leq \beta$.
\label{thm:DFHPRR}
\end{theorem}
We pair this together with the accuracy from either the Gaussian mechanism or the Laplace mechanism along with \Cref{cor:pop_acc} to get the following result

\begin{theorem}\label{thm:LapGauss}
Given confidence level $1-\beta$ and using the Laplace or Gaussian mechanism for each algorithm $\alg_i$, then $(\alg_1,\cdots, \alg_k)$ is $(\tol^{(1)}, \beta$)-accurate, where
\begin{itemize}
\item{\bf Laplace Mechanism}: We define $\tol^{(1)}$ to be the solution to the following program
{\begin{align*}
\min & \qquad \tol \\ 
\text{ s.t. } & \qquad \tol \geq \sqrt{\frac{48}{n} \ln\left(\frac{8}{\beta}\right)}  + \tau' \\ 
& \qquad \left(\tol - \tol' -  4 \epsilon' k  \cdot \left(\frac{e^{\epsilon'}-1}{e^{\epsilon'}+1} \right) \right)^2 \cdot \left(\tol - \tol' \right)   \geq 256 \eps'^2 k\ln\frac{16}{\beta} \\
\text{for } & \qquad \epsilon' > 0 \text{ and } \tau' = \frac{\ln\left(2k/\beta \right)}{n \epsilon' }
\end{align*}}

\item{\bf Gaussian Mechanism}: We define $\tol^{(1)}$ to be the solution to the following program
{\begin{align*}
\min & \qquad \tol \\ 
\text{ s.t. } & \qquad \tol \geq \sqrt{\frac{48}{n} \ln\left(\frac{8}{\beta}\right)} + \tol'  \\ 
& \qquad \left( \left(\tol - \tol' -  4 \rho k\right)^2 - 64\rho k \ln\sqrt{\pi \rho k} \right) \cdot \left(\tol - \tol' \right)   \geq 64\rho k \ln\frac{16}{\beta}  \\
\text{for } & \qquad \rho > 0 \text{ and } \tol' = \frac{1}{2n} \sqrt{\frac{1}{\rho}\ln(4k/\beta)} 
\end{align*}}
\end{itemize}
\end{theorem}

To bound the sample accuracy, we will use the following lemma that gives the accuracy guarantees of Laplace mechanism. 

\begin{lemma}
If $\{Y_i : i \in [k]\} \stackrel{i.i.d.}{\sim} \Lap(b)$, then for $\beta \in (0,1]$ we have:
\begin{align*}
 & \prob{|Y_i| \geq  \ln(1/\beta) b} \leq \beta   \implies \prob{\exists i \in [k] \text{ s.t. } |Y_i| \geq \ln(k/\beta) b} \leq \beta.
\label{eq:acc_lap} \numberthis
\end{align*}
\label{lem:accLap}
\end{lemma}

\begin{proof}[Proof of Theorem \ref{thm:LapGauss}]
We will focus on the Laplace mechanism part first, so that we add $\Lap\left( \frac{1}{n \epsilon'}\right)$ noise to each answer.  After $k$ adaptively selected queries, the entire sequence of noisy answers is $(\epsilon,\delta)$-DP where
\begin{align}
\epsilon = k \epsilon' \cdot \frac{e^{\epsilon'}-1}{e^{\epsilon'}+1} + \epsilon' \cdot \sqrt{2k \ln(1/\delta)} \label{eqn:eps_del}
\end{align}

Now, we want to bound the two terms in \eqref{eq:errors} by $\frac{\beta}{2}$ each.  We can bound sample accuracy as:
\begin{align*}
\tau' = \frac{1}{n \epsilon' }\ln\left(\frac{2k}{\beta}\right) 
\end{align*}
which follows from \Cref{lem:accLap}, and setting the error width to $\tau'$ and the probability bound to $\frac{\beta}{2}$.

For the population accuracy, we apply \Cref{thm:DFHPRR} to take a union bound over all selected statistical queries, and set the error width to $\tau - \tau'$ and the probability bound to $\frac{\beta}{2}$ to get:
{ \small
\begin{align*}
& \delta = \exp\left( \frac{-8\ln(16/\beta)}{\tau - \tau'} \right), \qquad \tau - \tau' \geq  \sqrt{\frac{48}{n} \ln\frac{8}{\beta}}, \qquad \text{ and } \qquad \tau - \tau' \geq  4 \epsilon \numberthis \label{eqn:tau's}
\end{align*}
}
We then use \eqref{eqn:eps_del} and write $\epsilon$ in terms of $\delta$ to get:
$$
\epsilon =  \epsilon' k \cdot \frac{e^{\epsilon'}-1}{e^{\epsilon'}+1} + 4 \epsilon' \cdot \sqrt{ \frac{k \ln(16/\beta)}{\tau - \tau'} }.  
$$

Substituting the value of $\epsilon$ in \Cref{eqn:tau's}, we get:
$$
\tau - \tau' \geq  4 \left(\epsilon' k  \cdot \left(\frac{e^{\epsilon'}-1}{e^{\epsilon'}+1} \right) + 4\epsilon'\cdot \sqrt{\frac{k\ln\frac{16}{\beta}}{\tau - \tau'}} \right)
$$ 

By rearranging terms, we get
$$\left(\tol - \tol' -  4 \epsilon' k  \cdot \left(\frac{e^{\epsilon'}-1}{e^{\epsilon'}+1} \right) \right)^2 \cdot \left(\tol - \tol' \right) \geq 256 \eps'^2 k\ln\frac{16}{\beta}$$

We are then left to pick $\epsilon'>0$  to obtain the smallest value of $\tol$.  

When can follow a similar argument when we add Gaussian noise with variance $\frac{1}{2n^2 \rho}$.  The only modification we make is using \Cref{thm:zCDP} to get a composed DP algorithm with parameters in terms of $\rho$, and the accuracy guarantee in \Cref{thm:acc}.
\end{proof}

\subsection{Confidence Bounds from \texorpdfstring{\cite{BNSSSU16}}{}}
\label{sec:bns_res}
We now go through the argument of \cite{BNSSSU16} to improve the constants as much as we can via their analysis to get a decent confidence bound on $k$ adaptively chosen statistical queries.  This requires presenting their \emph{monitoring}, which is similar to the monitor presented in \Cref{alg:monitor} but takes as input several independent datasets.  We first present the result.

\begin{theorem}
Given confidence level $1-\beta$ and using the Laplace or Gaussian mechanism for each algorithm $\alg_i$, then $(\alg_1,\cdots, \alg_k)$ is $(\tol^{(2)}, \beta$)-accurate.
\begin{itemize}
\item{\bf Laplace Mechanism}: We define $\tol^{(2)}$ to be the following quantity:
{\small
\begin{align*}
& \frac{1}{1-(1-\beta)^{\left\lfloor \frac{1}{\beta}\right\rfloor}}    \inf_{\substack{\epsilon'>0,\\ \delta \in (0,1)}} \left\{ e^\psi - 1 + 6\delta \left\lfloor \frac{1}{\beta}\right\rfloor +   \frac{\ln\frac{k}{2\delta}}{\epsilon' n}   \right\},   \text{where } \psi = \left(\frac{e^{\epsilon'}-1}{e^{\epsilon'}+1} \right) \cdot \epsilon' k + \epsilon'\sqrt{2k \ln\frac{1}{\delta}}
\end{align*}}
\item{\bf Gaussian Mechanism}: We define $\tol^{(2)}$ to be the following quantity:
{\small
\begin{align*}
& \frac{1}{1-(1-\beta)^{\left\lfloor \frac{1}{\beta}\right\rfloor}}  \inf_{\substack{\rho>0, \\\delta \in (0,1)}} \left\{  e^\xi - 1 + 6\delta \left\lfloor \frac{1}{\beta}\right\rfloor  +   \sqrt{\frac{\ln\frac{k}{\delta}}{n^2\rho} }   \right\}, \text{where } \xi = k \rho + 2 \sqrt{k \rho \ln\left(\frac{\sqrt{\pi \rho}}{\delta}\right)}
\end{align*}}
\end{itemize}
\label{thm:BNSSSU}
\end{theorem}


In order to prove this result, we begin with a technical lemma which considers an algorithm $\cW$ that takes as input a collection of $s$ samples and outputs both an index in $[s]$ and a statistical query, where we denote $\cQ_{SQ}$ as the set of all statistical queries $\phi: \cX \to [0,1]$ and their negation.

\begin{lemma}[\citep{BNSSSU16}]
Let $\cW: \left(\cX^n\right)^s \to \cQ_{SQ} \times [s]$ be $(\epsilon,\delta)$-DP.  If $\vec{X} = (X^{(1)}, \cdots, X^{(s)}) \sim \left(\cD^{n}\right)^s$ then 
$$
\left| \Ex{\vec{X},(\phi
,t) = \cW(\vec{X})}{\phi(\cD) - \phi(X^{(t)})} \right| \leq e^\epsilon -1 + s \delta
$$
\label{lem:technical}
\end{lemma}


We then define what we will call the \emph{extended monitor} in \Cref{alg:ex_monitor}.

\begin{algorithm}
\caption{Extended Monitor $\cW_\cD[\alg,\adv](\vX)$}
\label{alg:ex_monitor}
\begin{algorithmic}
\REQUIRE $\vX = (X^{(1)}, \cdots, X^{(s)}) \in (\cX^n)^s$
\FOR{$t  \in [s]$}
\STATE We simulate $\alg(X^{(t)})$ and $\adv$ interacting.  We write $\phi_{t,1}, \cdots, \phi_{t,k} \in \cQ_{SQ}$ as the queries chosen by $\adv$ and write $a_{t,1},\cdots, a_{t,k} \in \R$ as the corresponding answers of $\alg$.
\ENDFOR
\STATE Let 
$
(j^*,t^*) = \myargmax_{j \in [k], t \in [s]} \left| \phi_{t,j}(\cD) - a_{t,j} \right|.
$
\STATE \textbf{if} $a_{t^*,j^*} - \phi_{t^*,j^*}(\cD) \geq 0$ \textbf{then} $\phi^* \gets \phi_{t^*,j^*}$
\STATE \textbf{else} $\phi^* \gets -\phi_{t^*,j^*}$
\ENSURE $(\phi^*,t^*)$
\end{algorithmic}
\end{algorithm}
We then present a series of lemmas that leads to an accuracy bound from \cite{BNSSSU16}.
\begin{lemma}[\citep{BNSSSU16}]
For each $\epsilon,\delta\geq 0$, if $\alg$ is $(\epsilon,\delta)$-DP for $k$ adaptively chosen queries from $\cQ_{SQ}$, then for every data distribution $\cD$ and analyst $\adv$, the monitor $\cW_\cD[\alg,\adv]$ is $(\epsilon,\delta)$-DP.  
\end{lemma}
\begin{lemma}[\citep{BNSSSU16}]
If $\alg$ fails to be $(\tol,\beta)$-accurate, then $\phi^*(\cD) - a^* \geq 0$, where $a^*$ is the answer to $\phi^*$ during the simulation  ($\adv$ can determine $a^*$ from output $(\phi^*,t^*)$) and
\begin{align*}
& \Prob{\substack{\vec{X} \sim (\cD^n)^s, \\(\phi^*,t^*) = \cW_\cD [\alg,\adv](\vbX)}}{\left|\phi^*(\cD) - a^*\right|  > \tol }  > 1 - (1-\beta)^s.
\end{align*}
\label{lem:BNSSSU1}
\end{lemma}

The following result is not stated exactly the same as in \cite{BNSSSU16}, but it follows the same analysis.  We just do not simplify the expressions in the inequalities.  
\begin{lemma}
If $\alg$ is $(\tol',\beta')$-accurate on the sample but not $(\tol,\beta)$-accurate for the population, then 
\begin{align*}
& \left|\Ex{\substack{\vec{X} \sim (\cD^n)^{s},\\(\phi^*,t) = \cW_\cD[\alg,\adv](\vec{X})}}{\phi^*(\cD) - \phi^*\left(X^{(t)}\right)}
\right|   \geq \tol\left(1 - (1-\beta)^s\right) - \left(\tol' + 2 s \beta' \right).
\end{align*}
\label{lem:BNSSSUlb}
\end{lemma}

We now put everything together to get our result.

\begin{proof}[Proof of \Cref{thm:BNSSSU}]
We ultimately want a contradiction between the result given in \Cref{lem:technical} and \Cref{lem:BNSSSUlb}. Thus, we want to find the parameter values that minimizes $\tol$ but satisfies the following inequality
\begin{equation}
\tol\left(1 - (1-\beta)^s\right) - \left(\tol' + 2 s \beta' \right) > e^\epsilon -1 + s \delta.
\label{eq:contradict}
\end{equation}
We first analyze the case when we add noise $\Lap\left( \frac{1}{n\epsilon'}\right)$ to each query answer on the sample to preserve $\epsilon'$-DP of each query and then use advanced composition \Cref{thm:drv} to get a bound on $\epsilon$.
$$
\epsilon = \left(\frac{e^{\epsilon'}-1}{e^{\epsilon'} + 1}\right) \epsilon' k + \epsilon'\sqrt{2k \ln(1/\delta)} = \psi.
$$
Further, we obtain $(\tol',\beta')$-accuracy on the sample, where for $\beta' >0$ we have 
$
\tol' = \frac{\ln(k/\beta') }{\epsilon' n}.
$
We then plug these values into \eqref{eq:contradict} to get the following bound on $\tol$
\begin{align*}
\tol & \geq \left(\frac{1}{1 - (1-\beta)^s} \right) \left(  \frac{\ln\left(\frac{k}{\beta'}\right) }{\epsilon' n} +2s \beta'+ e^\psi-1 + s \delta \right)
\end{align*}
We then choose some of the parameters to be the same as in \cite{BNSSSU16}, like $s = \lfloor 1/\beta \rfloor$ and $\beta' = 2\delta$.  We then want to find the best parameters $\epsilon',\delta$ that makes the right hand side as small as possible.  Thus, the best confidence width $\tol$ that we can get with this approach is the following 
\begin{align*}
& \frac{1}{1-(1-\beta)^{\left\lfloor \frac{1}{\beta}\right\rfloor}} \cdot  \inf_{\substack{\epsilon'>0,\\ \delta \in (0,1)}} \left\{ e^\psi - 1 + 6\delta \left\lfloor \frac{1}{\beta}\right\rfloor  +   \frac{\ln\frac{k}{2\delta}}{\epsilon' n}   \right\}
\end{align*}

Using the same analysis but with Gaussian noise added to each statistical query answer with variance $\frac{1}{2\rho n^2}$ (so that $\alg$ is $\rho k$-zCDP), we get the following confidence width $\tol$,
\begin{align*}
& \frac{1}{1-(1-\beta)^{\left\lfloor \frac{1}{\beta}\right\rfloor}}  \inf_{\substack{\rho>0, \\\delta \in (0,1)}} \left\{  e^\xi - 1 + 6\delta \left\lfloor \frac{1}{\beta}\right\rfloor +   \sqrt{\frac{\ln\frac{k}{\delta}}{n^2\rho} }   \right\}
\end{align*}
\end{proof}

\subsection{Confidence Bounds for Thresholdout (\texorpdfstring{\cite{DFHPRR15nips}}{})}
\label{app:thresh_cw}

\begin{theorem}
If the Thresholdout mechanism $\cM$ with noise scale $\sigma$, and threshold $T$ is used for answering queries $\phi_i$, $i \in [k]$, with reported answers $a_1,\cdots, a_k$ such that $\cM$ uses the holdout set of size $h$ to answer at most $B$ queries, then
given confidence parameter $\beta$, Thresholdout is $(\tol, \beta$)-accurate, where 
{\small 
\begin{align*}
\tol = \sqrt{\frac{1}{\beta} \cdot \left( T^2 + \psi  + \frac{\xi}{4h} + \sqrt{\frac{\xi}{h} \cdot \left( T^2 + \psi\right)}\right) }
\end{align*}
}
for $\psi = \Ex{}{(\max\limits_{i \in [k]}W_{i} + \max\limits_{j\in [B]} Y_j)^2} + 2 T \cdot \Ex{}{\max\limits_{i \in [k]}W_{i} + \max\limits_{j\in [B]} Y_j}$, and $\xi = \min\limits_{\lambda \in [0,1)} \left(\frac{ \frac{2B}{\sigma^2h} - \ln \left( 1-\lambda \right)}{\lambda}\right)$, where $W_i \sim Lap(4\sigma), i \in [k]$ and $Y_j \sim Lap(2\sigma), j \in [B]$.
\label{thm:thresh_cw}
\end{theorem}

\begin{proof}
Similar to the proof of Theorem~\ref{thm:RZ_CW}, first we derive bounds on the mean squared error (MSE) for answers to statistical queries produced by Thresholdout.  We want to bound the maximum MSE over all of the statistical queries, where the expectation is over the noise added by the mechanism and the randomness of the adversary.

\begin{theorem}
If the Thresholdout mechanism $\cM$ with noise scale $\sigma$, and threshold $T$ is used for answering queries $\phi_i$, $i \in [k]$, with reported answers $a_1,\cdots, a_k$ such that $\cM$ uses the holdout set of size $h$ to answer at most $B$ queries, then we have 
\begin{align*}
 & \Ex{\substack{X \sim \cD^n, \\ \phi_{j^*} \sim \cW_\cD[\alg,\adv](X)}}{(a_{j^*}  - \phi_{j^*}(\cD))^2}  \leq T^2 + \psi  + \frac{\xi}{4h} + \sqrt{\frac{\xi}{h} \cdot \left( T^2 + \psi\right)},
\end{align*}
for $\psi = \Ex{}{(\max\limits_{i \in [k]}W_{i} + \max\limits_{j\in [B]} Y_j)^2} + 2 T \cdot \Ex{}{\max\limits_{i \in [k]}W_{i} + \max\limits_{j\in [B]} Y_j}$ and $\xi = \min\limits_{\lambda \in [0,1)} \left(\frac{ \frac{2B}{\sigma^2h} - \ln \left( 1-\lambda \right)}{\lambda}\right)$, where $W_i \sim Lap(4\sigma), i \in [k]$ and $Y_j \sim Lap(2\sigma), j \in [B]$.
\label{thm:thresh_mse}
\end{theorem}

\begin{proof}
Let us denote the holdout set in $\cM$ by $X_h$ and the remaining set as $X_t$. Let $\cO$ denote the distribution $\cW_\cD[\alg,\adv](X)$, where $X \sim \cD^n$. We have:
\begin{align*}
\Ex{\phi_{j^*} \sim \cO}{(a_{j^*}  - \phi_{j^*}(\cD))^2} &   = \Ex{\phi_{j^*} \sim \cO}{(a_{j^*}  - \phi_{j^*}(X_h) + \phi_{j^*}(X_h) - \phi_{j^*}(\cD))^2} \\
&   = \Ex{\phi_{j^*} \sim \cO}{(a_{j^*}  - \phi_{j^*}(X_h))^2}  + \Ex{\phi_{j^*} \sim \cO}{(\phi_{j^*}(X_h) - \phi_{j^*}(\cD))^2} \\
&  \qquad  + \Bigg( 2\sqrt{\Ex{\phi_{j^*} \sim \cO}{(a_{j^*}  - \phi_{j^*}(X_h))^2}} \cdot \sqrt{\Ex{\phi_{j^*} \sim \cO}{(\phi_{j^*}(X_h) - \phi_{j^*}(\cD))^2}}\Bigg) \label{eqn:all_diff_cw} \numberthis
\end{align*}
where the last inequality follows from the Cauchy-Schwarz inequality.

Now, define a set $S_h$ which contains the indices of the queries answered via $X_h$. We know that for at most $B$ queries $\phi_j \in S_h$, the output of $\cM$ was $a_j = \phi_j\left(X_h\right) + Z_j$ where $Z_j \sim Lap(\sigma)$, whereas it was $a_i = \phi_i\left(X_t\right)$ for at least $k - B$ queries, $i \in [k \setminus S_h]$. Also, define $W_i \sim Lap(4\sigma), i \in [k]$ and $Y_j \sim Lap(2\sigma), j \in S_h$. Thus, for any $j^* \in [k]$,  we have:

\begin{align*}
a_{j^*}  - \phi_{j^*}(X_h) &  \leq \max\left\{\max\limits_{i \in [k\setminus S_h]}|\phi_i(X_h) - \phi_i(X_t)|, \max\limits_{j \in S_h}Z_j\right\} \\
&   \leq \max\left\{\max\limits_{\substack{i \in [k\setminus S_h], \\j(i)\in S_h}}T + Y_{j(i)} + W_i, \max\limits_{j \in S_h}Z_j\right\} \\
&   \leq \max\left\{\max\limits_{\substack{i \in [k\setminus S_h], \\j\in S_h}}T + Y_{j} + W_i, \max\limits_{j \in S_h}Z_j\right\} \\
&   \leq T + \max\limits_{i \in [k]}W_{i} + \max\limits_{j\in [B]} Y_j
\end{align*}

Thus,
{
\begin{align*}
\Ex{\phi_{j^*} \sim \cO}{(a_{j^*}  - \phi_{j^*}(X_h))^2}  &  \leq \Ex{}{(T + \max\limits_{i \in [k]}W_{i} + \max\limits_{j\in [B]} Y_j)^2} \\
&  =  T^2 +  \Ex{}{(\max\limits_{i \in [k]}W_{i} + \max\limits_{j\in [B]} Y_j)^2}   + 2 T \cdot \Ex{}{\max\limits_{i \in [k]}W_{i} + \max\limits_{j\in [B]} Y_j} \label{eqn:f_term} \numberthis
\end{align*}
}

We bound the 2nd term in \eqref{eqn:all_diff_cw} as follows. For every $i \in S_h$, there are two costs induced due to privacy: the Sparse Vector component, and the noise addition to $\phi_i(X_h)$. By the proof of Lemma 23 in \cite{DFHPRR15nips}, each individually provides a guarantee of $\left(\frac{1}{\sigma h},0\right)$-DP. Using Theorem~\ref{thm:zCDP}, this translates to each providing a $\left(\frac{1}{2\sigma^2h^2}\right)$-zCDP guarantee. Since there are at most $B$ such instances of each, by Theorem~\ref{thm:zCDP} we get that $\cM$ is $\left(\frac{B}{\sigma^2h^2}\right)$-zCDP. Thus, by Lemma~\ref{lem:mutualCDP} we have
$$I\left(\cM(X_h);X_h\right) \leq \frac{B}{\sigma^2h}$$
Proceeding similar to the proof of \Cref{thm:RZ_MSE}, we use the sub-Gaussian parameter for statistical queries in \Cref{lem:SQgauss} to obtain the following bound from \Cref{thm:RZds}:
\begin{align*}
\Ex{\phi_{j^*} \sim \cO}{\left(\phi_{j^*}(X_h) - \phi_{j^*}(\cD) \right)^2} &  = \Ex{\substack{X \sim \cD^n,\\ \alg,\adv}}{\max_{i \in S_h} \left\{ (\phi_i(X_h) - \phi_i(\cD))^2 \right\} } \\
&   \leq \frac{\xi}{4h}  \numberthis \label{eqn:thr_exp}
\end{align*}

Defining $\psi = \Ex{}{(\max\limits_{i \in [k]}W_{i} + \max\limits_{j\in [B]} Y_j)^2} + 2 T \cdot \Ex{}{\max\limits_{i \in [k]}W_{i} + \max\limits_{j\in [B]} Y_j}$, and combining Equations \eqref{eqn:all_diff_cw}, \eqref{eqn:f_term}, and \eqref{eqn:thr_exp}, we get:
{\small
\begin{align*}
 \Ex{\phi_{j^*} \sim \cO}{(a_{j^*}  - \phi_{j^*}(\cD))^2} &\leq T^2 + \psi  + \frac{\xi}{4h} + \sqrt{\frac{\xi}{h} \cdot \left( T^2 + \psi\right)} 
\end{align*}
}
\end{proof}

We can use the MSE bound from Theorem~\ref{thm:thresh_mse}, and Chebyshev's inequality, to get the statement of the theorem.
\end{proof}

\section{Proofs}
\label{app:proofs}
Here, we provide the proofs that have been omitted from the main body of the paper.

\subsection{Proof of Theorem \ref{thm:RZ_CW}}
\label{app:proofRZ_CW}

Rather than use the stated result in \cite{RZ16}, we use a modified ``corrected'' version and provide a proof for it here. The result stated here and the one in \cite{RZ16} are incomparable.
\begin{theorem}
Let $\cQ_{\sigma}$ be the class of queries $\phi: \cX^n \to \R$ such that $\phi(X) - \phi(\cD^n)$ is $\sigma$-subgaussian where $X \sim \cD^n$.    If $\alg: \cX^n \to \cQ_{\sigma}$ is a randomized mapping from datasets to queries such that $\MutInfo{\alg(X)}{X} \leq B$ 
 then
\begin{align*}
& \Ex{\substack{X \sim \cD^n,\\ \phi \gets \alg(X)}}{\left(\phi(X) - \phi(\cD^n)) \right)^2}   \leq\sigma^2 \cdot \min\limits_{\lambda \in [0,1)} \left(\frac{ 2 B - \ln \left( 1-\lambda \right)}{\lambda}\right).
\end{align*}
\label{thm:RZds}
\end{theorem}

In order to prove the theorem, we need the following results.
\begin{lemma}[\cite{RZ16arxiv}, \cite{Gray90}]
Given two probability measures $P$ and $Q$ defined on a common measurable space and assuming that $P$ is absolutely continuous with respect to $Q$, then
$$
 \mathrm{D}_{KL}\left[ P || Q \right] = \sup_{X}\left\{\Ex{P}{X} - \log \Ex{Q}{\exp(X)} \right\}
$$
\label{lem:fact1}
\end{lemma}
\begin{lemma}[\cite{RZ16arxiv}]
If $X$ is a zero-mean subgaussian random variable with parameters $\sigma$ then
$$
\ex{\exp\left( \frac{\lambda X^2}{2 \sigma ^2}\right)} \leq \frac{1}{\sqrt{1 - \lambda}}, \qquad \forall \lambda \in [0,1)
$$
\label{lem:fact3}
\end{lemma}

\begin{proof}[Proof of Theorem~\ref{thm:RZds}]

Proceeding similar to the proof of Proposition 3.1 in \cite{RZ16arxiv}, we write $\pmb{\phi}(X) = (\phi(X): \phi \in \cQ_\sigma)$. We have:
{
\begin{align*}
\MutInfo{\alg(X)}{X}   &   \geq \MutInfo{\alg(X)}{\pmb{\phi}(X)} \\
& = \sum_{\bba, \phi \in \cQ_\sigma} \Bigg ( \ln\left(\frac{\prob{(\pmb{\phi}(X),\alg(X)) = (\bba,\phi)} }{\prob{\pmb{\phi}(X) = \bba}\prob{\alg(X) = \phi} } \right)  \cdot    \prob{(\pmb{\phi}(X),\alg(X)) = (\bba,\phi)}\Bigg )\\
& = \sum_{\bba, \phi \in \cQ_\sigma} \Bigg ( \ln\left(\frac{\prob{\pmb{\phi}(X) = \bba | \alg(X) = \phi} }{\prob{\pmb{\phi}(X) = \bba} } \right)  \cdot   \prob{\alg(X) = \phi}\prob{\pmb{\phi}(X) = \bba|\alg(X) = \phi}\Bigg ) \\
& \geq \sum_{a, \phi \in \cQ_\sigma} \Bigg ( \ln\left(\frac{\prob{\phi(X) = a | \alg(X) = \phi} }{\prob{\phi(X) = a} } \right)   \cdot  \prob{\alg(X) = \phi}  \prob{\phi(X) = a|\alg(X) = \phi}\Bigg ) \\
& = \sum_{\phi \in \cQ_{\sigma}} \big ( \prob{\alg(X) = \phi}    \cdot  \mathrm{D}_{KL}\left[ (\phi(X) | \alg(X) = \phi) || \phi(X) \right] \big )
\numberthis \label{eqn: i_T}
\end{align*}}
where the first inequality follows from post processing of mutual information, i.e. the data processing inequality.  Consider the function $f_\phi(x) = \frac{\lambda}{2 \sigma^2} (x - \phi(\cD^n) )^2$ for $\lambda \in [0,1)$. We have
{ 
\begin{align*}
\mathrm{D}_{KL}\left[ (\phi(X) | \alg(X) = \phi) || \phi(X) \right] &   \geq \Ex{X \sim \cD^n, \alg}{f_\phi(\phi(X)) | \alg(X) = \phi} - \ln \Ex{\substack{X \sim \cD^n, \\ \phi \sim \alg(X)}}{ \exp\left( f_\phi(\phi(X)) \right) } \\
&  \geq \frac{\lambda}{2 \sigma^2} \Ex{X \sim \cD^n,\alg}{\left(\phi(X) - \phi(\cD^n) \right)^2 | \alg(X) = \phi}   - \ln \left(\frac{1}{\sqrt{1- \lambda}}\right)  
\end{align*}}
where the first and second inequalities follows from Lemmas~\ref{lem:fact1} and \ref{lem:fact3}, respectively.

Therefore, from \cref{eqn: i_T}, we have
{
\begin{align*}
 \MutInfo{\alg(X)}{X} &  \geq \frac{\lambda}{2 \sigma^2} \Ex{\substack{X \sim \cD^n, \\ \phi \sim \alg(X)}}{\left(\phi(X) - \phi(\cD^n) \right)^2}   - \ln \left(\frac{1}{\sqrt{1- \lambda}}\right) 
\end{align*}}

Rearranging terms, we have
\begin{align*}
\Ex{\substack{X \sim \cD^n, \\ \phi \sim \alg(X)}}{\left(\phi(X) - \phi(\cD) \right)^2} &   \leq \frac{2 \sigma^2}{\lambda}\left(\MutInfo{\alg(X)}{X}+ \ln \left(\frac{1}{\sqrt{1- \lambda}}\right) \right) \\
&  = \sigma^2 \cdot \frac{2\MutInfo{\alg(X)}{X}- \ln \left(1- \lambda\right)}{\lambda} 
\end{align*}
\end{proof}

In order to apply this result, we need to know the subgaussian parameter for statistical queries and the mutual information for private algorithms.  

\begin{lemma}
For statistical queries $\phi$ and $X \sim \cD^n$, we have $\phi(X) - \phi(\cD^n)$ is $\frac{1}{2 \sqrt{n}}$-sub-gaussian.
\label{lem:SQgauss}
\end{lemma}

We also use the following bound on the mutual information for zCDP mechanisms: 
\begin{lemma}[\cite{BS16}]
If $\alg: \cX^n \to \cY$ is $\rho$-zCDP and $X \sim \cD^n$, then 
$
\MutInfo{\alg(X)}{X} \leq \rho n.
$
\label{lem:mutualCDP}
\end{lemma}

In order to prove \Cref{thm:RZ_MSE}, we use the same monitor from \Cref{alg:monitor} in which there is a single dataset as input to the monitor and it outputs the query whose answer had largest error with the true query answer.  We first need to show that the monitor has bounded mutual information as long as $\alg$ does, which follows from mutual information being preserved under post-processing.
\begin{lemma}
If $\MutInfo{X}{\alg(X)} \leq B $ where $X \sim \cD^n$, then $\MutInfo{X}{\cW_\cD[\alg,\adv](X)} \leq B$.  
\label{lem:compMutInfo}
\end{lemma}

Next, we derive bounds on the mean squared error (MSE) for answers to statistical queries produced by the Gaussian mechanism.  We want to bound the maximum MSE over all of the statistical queries, where the expectation is over the noise added by the mechanism and the randomness of the adversary.
{
\begin{align*}
\Ex{\substack{X \sim \cD^n,\\ \adv, \alg}}{\max_{ i \in [k] }\left(\phi_i(\cD) - a_i \right)^2}  &   \leq 2   \Ex{\substack{X \sim \cD^n,\\\adv, \alg}}{\max_{ i \in [k] }\left\{\left(\phi_i(\cD) - \phi_i(X)\right)^2 + \left( \phi_i(X) - a_i \right)^2\right\}} \nonumber\\
&  = 2 \cdot \ex{\max_{i \in [k]}\left( \phi_i(\cD) - \phi_i(X)\right)^2}  + 2 \cdot  \Ex{Z_i \sim N\left(0,\frac{1}{2n^2 \rho}\right) }{\max_{i \in [k]} Z_i^2}
\label{eq:MSE_bound} \numberthis
\end{align*}}

To bound $\ex{\max_{i \in [k]}\left( \phi_i(\cD) - \phi_i(X)\right)^2} $, we obtain the following  using the monitor argument from \cite{BNSSSU16} 
along with results from \cite{RZ16, BS16}.
\begin{theorem}
For parameter $\rho \geq 0$, the answers provided by the Gaussian mechanism $a_1,\cdots, a_k$ against an adaptively selected sequence of queries satisfy:
{
\begin{align*}
\Ex{\substack{X \sim \cD^n,\\ \alg,\adv}}{\max_{i \in [k]}\left(\phi_i(\cD) - a_i \right)^2} &   \leq \frac{1}{2n} \cdot \min\limits_{\lambda \in [0,1)} \left(\frac{ 2 \rho k n - \ln \left( 1-\lambda \right)}{\lambda}\right)  +2 \cdot  \Ex{Z_i \sim N\left(0,\frac{1}{2n^2 \rho}\right) }{\max_{i \in [k]} Z_i^2}
\end{align*}
}
\label{thm:RZ_MSE}
\end{theorem}

\begin{proof}
We follow the same analysis for proving \Cref{thm:BNSSSU} where we add Gaussian noise with variance $\frac{1}{2 \rho n^2}$ to each query answer so that the algorithm $\cM$ is $\rho$-zCDP, which (using \Cref{lem:mutualCDP} and the post-processing property of zCDP) makes the mutual information bound $B = \rho k n$.  We then use \Cref{lem:compMutInfo} and the sub-Gaussian parameter for statistical queries in \Cref{lem:SQgauss} to obtain the following bound from \Cref{thm:RZds}.
\begin{align*}
\Ex{\substack{X \sim \cD^n,\\ \phi^* \sim \cW_\cD[\alg,\adv](X)}}{\left(\phi^*(X) - \phi^*(\cD)) \right)^2} &  = \Ex{X \sim \cD^n, \alg,\adv}{\max_{i \in [k]} \left\{ (\phi_i(X) - \phi_i(\cD))^2 \right\} \}} \\
&  \leq \frac{1}{4n} \cdot \min\limits_{\lambda \in [0,1)} \left(\frac{ 2 \rho k n - \ln \left( 1-\lambda \right)}{\lambda}\right) \label{eqn:pop_mse} \numberthis
\end{align*}

We then combine this result with \eqref{eq:MSE_bound} to get the statement of the theorem.
\end{proof}

\begin{proof}[Proof of Theorem \ref{thm:RZ_CW}]
We want to bound the two terms in \eqref{eq:errors} by $\frac{\beta}{2}$ each.  We start by bounding the sample accuracy via the following constraint:
\begin{align*}
\tol' \geq \frac{1}{2n} \sqrt{\frac{1}{\rho}\ln(4k/\beta)} 
\end{align*}
which follows from \Cref{thm:acc}, and setting the error width to $\tau'$ and the probability bound to $\frac{\beta}{2}$.

Next, we can bound the population accuracy in \eqref{eq:errors} using \Cref{eqn:pop_mse} and Chebyshev's inequality to obtain the following high probability bound,
\begin{align*}
& \Prob{\substack{X \sim \cD^n,\\ \phi \gets \alg(X)}}{|\phi(X) - \phi(\cD)| \geq \tol - \tau'}      \leq \frac{1}{4n (\tau - \tau')^2 } \cdot \min\limits_{\lambda \in [0,1)} \left(\frac{ 2 \rho k n - \ln \left( 1-\lambda \right)}{\lambda}\right) 
\end{align*}
which implies that for bounding the probability by  $\frac{\beta}{2}$, we get
\begin{align*}
\tau - \tau' \geq \sqrt{\frac{1}{2n \beta } \cdot \min\limits_{\lambda \in [0,1)} \left(\frac{ 2 \rho k n - \ln \left( 1-\lambda \right)}{\lambda}\right)}
\end{align*}

We then use the result of \Cref{cor:pop_acc} to obtain our accuracy bound.
\end{proof}

\subsubsection{Comparison of Theorem \ref{thm:RZ_CW} with Prior Work}
\label{app:comp}
One can also get a high-probability bound on the sample accuracy of $\alg(X)$ using Theorem 3 in \cite{XR17}, resulting in 
\begin{align}
\tol =  \sqrt{\frac{2}{n} \left(\frac{2 \rho k n}{\beta} + \log\left(\frac{4}{\beta} \right) \right)} + \frac{1}{2n} \sqrt{\frac{1}{\rho}\ln\left(\frac{4k}{\beta}\right)} \label{eqn:xr}
\end{align}
where i.i.d. Gaussian noise $N\left(0, \frac{1}{2\rho n^2}\right)$ has been added to each query. The proof is similar to the proof of Theorem \ref{thm:RZ_CW}.
If the mutual information bound $B = \rho kn \geq 1$, then the first term in the expression of the confidence width in \Cref{thm:RZ_CW} is less than the first term in   \cref{eqn:xr}, thus making Theorem~\ref{thm:RZ_CW} result in a tighter bound for any $\beta \in (0,1)$. For very small values of $B$, there exist sufficiently small $\beta$ for which the result obtained via \cite{XR17} is better.

\subsection{RMSE analysis for the single-adaptive query strategy}
\label{app:strategy}
\begin{theorem} \label{thm:strat}
The output by the single-adaptive query strategy above results in the maximum possible RMSE for an adaptively chosen statistical query when each sample in the dataset is drawn uniformly at random from $\{-1,1\}^{k+1}$, and $\cM$ is the Naive Empirical Estimator, i.e., $\cM$ provides the empirical correlation of each of the first $k$ features with the $(k+1)^{th}$ feature.
\end{theorem}
\begin{proof}
Consider a dataset $X \in \cX^n$, where $\cX$ is the uniform distribution over $\{-1,1\}^{k+1}$. We will denote the $j^{th}$ element of $x_i \in X$ by $x_i(j)$, for $j \in [k+1]$. Now, $\forall j \in [k]$, we have that:
{
\begin{align*}
\Ex{X}{\frac{1 + x(j) \cdot x(k+1)}{2}} &  = \frac{1 + \Pr \left(x(j) = x(k+1)\right)  - \Pr\left(x(j)\neq x(k+1)\right)}{2} = a_j
\end{align*}
}
\begin{align*}
 & \therefore \Pr_X \left(x(j) = x(k+1)\right)  = a_j \qquad  \text{and}  \qquad \Pr_X \left(x(j) \neq x(k+1)\right) = 1 - a_j
\end{align*}
Now,
{
\begin{align*}
\ln\left(\frac{\Pr_X \left(x(k+1) = 1 | \land_{j \in [k]} x(j) = x_j \right)}{\Pr_X \left(x(k+1)
 = -1 | \land_{j \in [k]} x(j) = x_j \right)}\right) &  = \ln\left(\frac{\Pr_X \left(x(k+1) = 1 \land ( \land_{j \in [k]} x(j) = x_j) \right)}{\Pr_X \left(x(k+1) = -1\land (  \land_{j \in [k]} x(j) = x_j) \right)}\right) \\
&  = \ln\left(\prod\limits_{j \in [k] } \frac{\Pr_X \left(x(k+1) = 1 \land  x(j) = x_j \right)}{\Pr_X \left(x(k+1) = -1 \land x(j) = x_j \right)}\right) \\
&  = \ln\left(\prod\limits_{j \in [k] } \left( \frac{\Pr_X \left(x(k+1) =  x(j)  \right)}{\Pr_X \left(x(k+1) \neq x(j)  \right)} \right)^{x_j}\right) \\
&  = \ln\left(\prod\limits_{j \in [k] } \left( \frac{a_j}{1 - a_j} \right)^{x_j}\right) \\
&  = \sum\limits_{j \in [k] }\left(x_j \cdot \ln  \frac{a_j}{1 - a_j} \right)
\end{align*}}
Thus,
{
$$
\phi_{k+1}(x) = \frac{sign\left(\ln\left(\frac{\Pr_X \left(x(k+1) = 1 | \land_{j \in [k]} x(j) = x_j \right)}{\Pr_X \left(x(k+1) = -1 | \land_{j \in [k]} x(j) = x_j \right)}\right)  \right) + 1}{2}$$}

As a result, the adaptive query $\phi_{k+1}$ in Algorithm~\ref{alg:strategy} (setting input $S = \{k+1\}$) corresponds to a naive Bayes classifier of $x(k+1)$, and given that $\cX$ is the uniform distribution over $\{-1,1\}^{k+1}$, this is the best possible classifier for $x(k+1)$. This results answer $a_{k+1}$ achieving the maximum possible deviation from the answer on the population, which is 0.5 as $\cX$ is uniformly distributed over $\{-1,1\}^{k+1}$. Thus, $a_{k+1}$ results in the maximum possible RMSE.
\end{proof}

\subsection{Proof of Theorem \ref{thm:gnc_acc}}
\label{app:proofGNCmain}

We start by proving the validity for query responses output by GnC that correspond to the responses provided by $\cM_g$, i.e., each query $\phi_i$ s.t. the output of GnC is $(a_{g, i}, \tau_i)$. 

\begin{lemma}
If the function $HoldoutTol(\beta', a_g, a_h) = \sqrt{\frac{\ln{2/\beta'}}{2n_h}}$ in GnC (Algorithm~\ref{alg:gnc}), 
then for each query $\phi_i$ s.t. the output of GnC is $(a_{g, i}, \tau_i)$, we have $\Pr{\left(|a_{g, i} - \phi_i(\cD)|>\tau_i\right)}\leq \beta_i.$
\label{lem:gnc_width1}
\end{lemma}
\begin{proof}
Consider a query $\phi_i$ for which the output of the GnC mechanism is $(a_{g, i}, \tau_i)$, and let $\tau_h = HoldoutTol(\beta_i, a_{g, i}, a_{h, i})$. Now, we have
\begin{align*}
\Pr{\left(|a_{g, i} - \phi_i(\cD)|>\tau_i\right)} &  \leq \Pr{\left(|a_{g, i} - a_{h, i}| + |a_{h, i} - \phi_i(\cD)|>\tau_i\right)} \\
&  = \Pr{\left(|a_{h, i} - \phi_i(\cD)|>\tau_h \right)} \\
&  \leq \beta_i
\end{align*}
where the equality follows since $|a_{g, i} - a_{h, i}| \leq \tau_i - \tau_h$, and the last inequality follows from applying the Chernoff bound for statistical queries.
\end{proof}

Next, we provide the accuracy for the query answers output by GnC that correspond to discretized empirical answers on the holdout. It is obtained by maximizing the discretization parameter such that applying the Chernoff bound on the discretized answer satisfies the required validity guarantee. 
\begin{lemma}
If failure $f$ occurs in GnC (Algorithm~\ref{alg:gnc}) for query $\phi_i$ and the output of GnC is $\left(\lfloor a_{h, i}\rfloor_{\gamma_f}, \tau_i\right)$,  since we have $\gamma_f  = \max\limits_{[0,\tau')} \gamma \text{ s.t. } 2e^{-2(\tau' - \gamma)^2n_h}\leq \beta',$ we have $\Pr{\left(|a_{g, i} - \phi_i(\cD)|>\tau_i\right)}\leq \beta_i.$ Here, $\lfloor y\rfloor_{\gamma}$ denotes $y$ discretized to multiples of $\gamma$
\label{lem:gnc_disc}
\end{lemma}

\begin{proof}[Proof of Theorem \ref{thm:gnc_acc}]
Let an instance of the Guess and Check mechanism $\cM$ encounter $f$ failures while providing responses to $k$ queries $\{\phi_1, \ldots, \phi_k\}$.
We will consider the interaction between an analyst $\cA$ and the Guess and Check mechanism $\cM$ to form a tree $T$, where the nodes in $T$ correspond to queries, and each branch of a node is a possible answer for the corresponding query. We first note a property about the structure of $T$:

\textbf{Fact 1:} For any query $\phi_{i'}$, if the check within $\cM$ results in failure $f'$, then there are $\frac{1}{\gamma_{f'}}$ possible responses for $\phi_{i'}$. On the other hand, if the check doesn't result in  a failure, then there is only 1 possible response for $\phi_{i'}$, namely $(a_{g, i'}, \tau_{i'})$.

Next, notice that each node in $T$ can be uniquely identified by the tuple $t = (i', f', \{j_1, \ldots, j_{f'}\}, \{\gamma_{j_1}, \ldots, \gamma_{j_{f'}}\})$, where $i'$ is the depth of the node (also, the index of the next query to be asked), $f'$ is the number of failures within $\cM$ that have occurred in the path from the root to node $t$, and for $\ell \in [f']$, the value $j_\ell$ denotes the query index of the $\ell$th failure on this path, whereas $\gamma_{j_\ell}$ is the corresponding discretization parameter that was used to answer the query. We can now observe another property about the structure of $T$:

\textbf{Fact 2:} For any $i'\in [k], f' \in [i'-1]$, there are    $\binom{i'-1}{f'}  \prod_{\ell \in [f']} \left( \frac{1}{\gamma_{j_{\ell}}}\right)$ nodes in $T$ of type $(i', f', ; , ;)$. This follows since there are $\binom{i'-1}{f'}$ possible ways that $f'$ failures can occur in $i'-1$ queries, and from Fact 1 above, there are $\frac{1}{\gamma_{j_\ell}}$ possible responses for a failure occurring at query index $j_\ell, \ell \in [f']$.

Now, we have
{
\begin{align*}
\Pr{\left(\exists i \in [k]: | \phi_i(\cD) - a_i  | >  \tol_i\right)} &  \leq \sum\limits_{\text{node  } t \in T} \Pr{\left(| \phi_{t}(\cD) - a_{t}  | >  \tol_{t}\right)} \\
&  = \Bigg( \sum\limits_{i' \in [k]} \sum\limits_{f' \in [i' - 1]} \sum\limits_{\{j_1, \ldots, j_{f'}\}} \sum\limits_{\{\gamma_{j_1}, \ldots, \gamma_{j_{f'}}\}}  	\Pr{\left(| \phi_{i'}(\cD) - a_{i'}  | >  \tol_{i'}| t \right)} \Bigg) \\
& = \sum\limits_{i' \in [k]} \sum\limits_{f' \in [i' - 1]} \sum\limits_{\{j_1, \ldots, j_{f'}\}} \sum\limits_{\{\gamma_{j_1}, \ldots, \gamma_{j_{f'}}\}} \frac{\beta \cdot c_{i'-1} \cdot c_{f'}}{\nu_{i',f',\gamma_{j_1}^{j_{f'}}}} \\
&  = \beta \Bigg(  \sum\limits_{i' \in [k]} \sum\limits_{f' \in [i' - 1]} \sum\limits_{\{j_1, \ldots, j_{f'}\}} \sum\limits_{\{\gamma_{j_1}, \ldots, \gamma_{j_{f'}}\}}  \frac{c_{i'-1} \cdot c_{f'}}{\binom{i'-1}{f'}  \prod_{\ell \in [f']} \left( \frac{1}{\gamma_{j_{\ell}}}\right)} \Bigg) \\
&  = \beta \left(\sum\limits_{i' \in [k]}  c_{i'-1} \cdot \left(\sum\limits_{f' \in [i' - 1]} c_{f'}\right)\right) \leq \beta \sum\limits_{i' \in [k]}  c_{i'-1} \\
&  \leq \beta
\end{align*}}
where the second equality follows from Lemma~\ref{lem:gnc_width1} (equivalently, Lemma~\ref{lem:gnc_width_mgf}), Lemma~\ref{lem:gnc_disc}, and substituting the values of $\beta_i$ in Algorithm~\ref{alg:gnc}; the last equality follows from Fact 2 above; and the last two inequalities follow since $\sum_{j\geq 0} c_j \leq 1$. Thus, we have simultaneous coverage $1-\beta$ for the Guess and Check mechanism $\cM$.
\end{proof}

\Cref{lem:gnc_width1} is agnostic to the guesses and holdout answers while computing the holdout tolerance $\tau_h$. However, GnC can provide a better tolerance $\tau_h$ in the presence of low-variance queries. We provide a proof for it below.

\subsection{Proof of \texorpdfstring{\Cref{lem:gnc_width_mgf}}{}}
\label{app:proofGNClem} 

\Cref{lem:gnc_width_mgf} uses the Moment Generating Function (MGF) of the binomial distribution to approximate the probabilities of deviation of the holdout's empirical answer from the true population mean (instead of, say, optimizing parameters in a large deviation bound). This is exact when the query only takes values in $\{0,1\}$. To prove the lemma, we start by first proving the dominance of the binomial MGF.

\begin{lemma}\label{lem:bin_mgf}
Let $X_1,X_2, ..., X_n$ be i.i.d. random variables in $[0,1]$, distributed according to $\cD$, and let $\mu = \Ex{}{X_i}$. Let $S=\sum_{i=1}^n X_i$, and $B\sim B(n,\mu)$ be a binomial random variable. Then, we have:
\begin{equation*}
\Pr(S > t) \leq \min\limits_{\lambda > 0} \frac{\Ex{}{e^{\lambda B}}}{e^{\lambda t}}.
\end{equation*}
\end{lemma}
\begin{proof}
Consider some $\lambda > 0$. We have
\begin{align*}
\Ex{}{e^{\lambda S}} & = \left(\Ex{}{e^{\lambda X_1}}\right)^n \\
& \leq  \left(\Ex{}{X_1 \cdot e^{\lambda } + (1 - X_1) \cdot e^0}\right)^n \\
& = \left(1 + \Ex{}{X_1}(e^{\lambda } - 1)\right)^n \\
&   = \left( 1 + \mu (e^{\lambda } - 1) \right)^n \\
& = \Ex{}{e^{\lambda B}} \numberthis\label{eqn:mgf_sum}
\end{align*}
where the first equality follows since $X_1,X_2, ..., X_n$ are i.i.d., and the last equality represents the MGF of the binomial distribution.

Now, we get:
\begin{align*}
\Pr(S > t) & = \Pr\left(e^{\lambda S} > e^{\lambda t}\right) \leq \frac{\Ex{}{e^{\lambda S}}}{e^{\lambda t}} \leq \frac{\Ex{}{e^{\lambda B}}}{e^{\lambda t}} \\
& \leq \min\limits_{\lambda' > 0} \frac{\Ex{}{e^{\lambda' B}}}{e^{\lambda' t}}
\end{align*}
where the first inequality follows from the Chernoff bound, and the second inequality follows from \Cref{eqn:mgf_sum}.
\end{proof}

\begin{proof}[Proof of \Cref{lem:gnc_width_mgf}]
Consider a query $\phi_i$ for which the output of the GnC mechanism is $(a_{g, i}, \tau_i)$. Let $\tau_h = HoldoutTol(\beta_i, a_{g, i}, a_{h, i})$.  For proving $\Pr{\left(|a_{g, i} - \phi_i(\cD)|>\tau_i\right)}\leq \beta_i$, it suffices to show that if $|a_{g, i} - \phi(\cD)|> \tau_i $, then 
\begin{align*}
\sup\limits_{\substack{\cD \text{ s.t. }\\ \phi_i(\cD) = a_{g, i} - \tau_i}}\Pr\limits_{X_h \sim \cD^{n_h}}{(a_{h, i} \geq a_{g, i} - \tau_i + \tau_h)} \leq \frac{\beta_i}{2} \numberthis \label{eqn:w_l2}\\
\text{ and }  \sup\limits_{\substack{\cD \text{ s.t. }\\ \phi_i(\cD) = a_{g, i} + \tau_i}}\Pr\limits_{X_h \sim \cD^{n_h}}{(a_{h, i} \leq a_{g, i} + \tau_i - \tau_h)} \leq \frac{\beta_i}{2} \numberthis\label{eqn:w_r2}
\end{align*}

When $a_{g, i} > a_{h, i}$, we only require inequality~\ref{eqn:w_l2} to hold. Let $B\sim B(n,\mu)$ be a binomial random variable. We have: 
\begin{align*}
\sup\limits_{\substack{\cD \text{ s.t. }\\ \phi_i(\cD) = \mu}}\Pr\limits_{X_h \sim \cD^{n_h}}{(a_{h, i} \geq \mu + \tau')}  &  \leq \min\limits_{\lambda > 0} \frac{\Ex{}{e^{\lambda B}}}{e^{\lambda n (\mu + \tau')}} \\
&  = \min\limits_{\lambda > 0} e^{ \left\{ \ln{\left(\Ex{}{e^{\lambda B}}\right)} - \lambda n (\mu + \tau')\right\}} \\
&  =  \frac{\Ex{}{e^{\ell B}}}{e^{\ell n (\mu + \tau')}} \\
&  = \frac{\left( 1 + \mu (e^{\ell} - 1) \right)^n}{e^{\ell n (\mu + \tau')}} \numberthis \label{eqn:mgf_low}
\end{align*}
where $\ell=\arg\min\limits_{\lambda > 0} e^{ \left\{ \ln{\left(\Ex{}{e^{\lambda B}}\right)} - \lambda n (\mu + \tau')\right\}}$, i.e., $ \frac{\mu e^\ell}{1 + \mu(e^\ell + 1)} = \mu + \tau' $. Here, the first inequality follows from \Cref{lem:bin_mgf} by setting $t = (\mu + \tau')n$, and the last equality follows from the MGF of the binomial distribution. Thus, we get that inequality~\ref{eqn:w_l2} holds for $\mu = a_{g, i} - \tau_i$.

 Similarly, when $a_{g, i} \leq a_{h, i}$, we only require inequality~\ref{eqn:w_r2} to hold. Let $\mu' = 1 - \mu$, and $B'\sim B(n,\mu')$. Therefore, we get
 \begin{align*}
\sup\limits_{\substack{\cD \text{ s.t. }\\ \phi_i(\cD) = \mu}}\Pr\limits_{X_h \sim \cD^{n_h}}{(a_{h, i} \leq \mu - \tau')} &  = \sup\limits_{\substack{\cD \text{ s.t. }\\ \phi_i(\cD) = \mu'}}\Pr\limits_{X_h \sim \cD^{n_h}}{(a_{h, i} \geq \mu' + \tau')} \\
&   \leq \frac{\left( 1 + \mu' (e^{\ell'} - 1) \right)^n}{e^{\ell' n (\mu' + \tau')}} 
\end{align*}
where $\ell'=\arg\min\limits_{\lambda > 0} e^{ \left\{ \ln{\left(\Ex{}{e^{\lambda B'}}\right)} - \lambda n (\mu' + \tau')\right\}}$, i.e., $ \frac{\mu' e^{\ell'}}{1 + \mu'(e^{\ell'} + 1)} = \mu' + \tau' $. Here, the inequality follows from inequality~\ref{eqn:mgf_low}. Thus, we get that inequality~\ref{eqn:w_r2} holds  for $\mu = a_{g, i} + \tau_i$.

\end{proof}

\newpage
\section{Pseudocodes} 
\label{app:codes}

\begin{algorithm}
\caption{Thresholdout (\cite{DFHPRR15nips})}
\label{alg:thresh}
\begin{algorithmic}
\REQUIRE train size $t$, threshold $T$, noise scale $\sigma$
\STATE Randomly partition dataset $X$ into a training set $X_{t}$ containing $t$ samples, and a holdout set $X_h$ containing $h = n-t$ samples
\STATE Initialize $\hat{T} \gets T + Lap(2\sigma)$
\FOR{each query $\phi$}
\IF{$|\phi(X_h) - \phi(X_t)| > \hat{T} + Lap(4\sigma)$}
\STATE $\hat{T} \gets T + Lap(2\sigma)$
\STATE Output $\phi(X_h) + Lap(\sigma)$
\ELSE
\STATE Output $\phi(X_t)$
\ENDIF
\ENDFOR

\end{algorithmic}
\end{algorithm}

\begin{algorithm}
\caption{A custom adaptive analyst strategy for random data}
\label{alg:strategy}
\begin{algorithmic}
\REQUIRE Mechanism $\cM$ with a hidden dataset $X \in \{-1, 1\}^{n \times (k+1)}$, set $S \subseteq [k + 1]$ denoting the indices of adaptive queries\footnotemark
\STATE Define 
$j \leftarrow 1$, and 
$success\leftarrow True$
\WHILE{$j \leq k$ and $success=True$}
\IF{$j \in S$}
\STATE Define $\phi_{j}(x) = \frac{sign\left(\sum\limits_{i \in [j-1]\setminus S}\left(x(i) \cdot \ln\frac{a_i}{a_i - 1}\right)\right) + 1}{2}$, where $
sign(y) = \begin{cases}
             1  & \text{if } y \geq 0\\
             -1  & \text{otherwise}
       \end{cases}$
\ELSE
\STATE Define $\phi_j(x) = \frac{1 + x(j) \cdot x(k+1)}{2}$ \ENDIF
\STATE Give $\phi_j$ to $\cM$,
and receive $a_j \in [0,1] \cup \bot$ from $\cM$
\IF{$a_j = \bot$}
\STATE $success = False$
\ELSE
\STATE $j \leftarrow j + 1$
\ENDIF
\ENDWHILE
\end{algorithmic}
\end{algorithm}
\footnotetext{For the single-adaptive query strategy used in the plots in Figure~\ref{fig:two_round}, we set $S=\{k+1\}$. For the quadratic-adaptive strategy used in the plots in Section~\ref{sec:gnc}, we set $S = \{i : 1 < i \leq k \text{ and } \exists \ell \in \mathbb{N} \text{ s.t. } \ell<i \text{ and } \ell^2 = i \}$.}

\fi

\end{document}